\pdfoutput=1
\documentclass[11pt]{article}

\usepackage[utf8]{inputenc}
\usepackage[T1]{fontenc}
\usepackage[margin=1in]{geometry}
\usepackage[hidelinks,pdfauthor={},pdftitle={},pdfsubject={},pdfkeywords={}]{hyperref}
\usepackage{url}
\usepackage{booktabs}
\usepackage{amsfonts}
\usepackage{nicefrac}
\usepackage{microtype}
\usepackage{amsmath, amssymb, amsthm}
\usepackage{graphicx}
\usepackage{xcolor}
\usepackage{natbib}
\newcommand{\us}{\_\allowbreak}
\newcommand{\slsh}{/\allowbreak}

\newtheorem{theorem}{Theorem}[section]
\newtheorem{lemma}[theorem]{Lemma}
\newtheorem{proposition}[theorem]{Proposition}
\newtheorem{corollary}[theorem]{Corollary}
\newtheorem{remark}[theorem]{Remark}
\theoremstyle{definition}
\newtheorem{definition}[theorem]{Definition}

\newcommand{\E}{\mathbb{E}}
\newcommand{\R}{\mathbb{R}}
\newcommand{\Sig}{\boldsymbol{\Sigma}}

\newcommand{\relDM}{\mathrm{rel}\text{-}d_M}
\newcommand{\sstar}{s^{\ast}}

\title{The Precision Wall: An Optimizer-Dependent Boundary for\\
  Optimization under Power-Law Measurement Spectra}

\author{Lngel Kyo\\
  A2O Labs\\
  \texttt{lngel.kyo@a2o-labs.com}}
\date{}

\newcommand{\deltaAdam}{0.48}
\newcommand{\deltaAdamThree}{0.482}
\newcommand{\rmsCalibB}{0.013}
\newcommand{\rmsCalibA}{0.042}
\newcommand{\deltaCalibA}{0.87}
\newcommand{\predEightFour}{5.20}
\newcommand{\onsetEightFour}{5.21}
\newcommand{\absdEightFour}{0.01}
\newcommand{\predSixteenEight}{3.71}
\newcommand{\onsetSixteenEight}{3.70}
\newcommand{\absdSixteenEight}{0.01}
\newcommand{\predThirtyTwoEight}{3.18}
\newcommand{\onsetThirtyTwoEight}{3.19}
\newcommand{\absdThirtyTwoEight}{0.01}
\newcommand{\predSixteenFour}{4.25}
\newcommand{\onsetSixteenFour}{4.23}
\newcommand{\absdSixteenFour}{0.02}
\newcommand{\maxAbsDelta}{0.02}
\newcommand{\hoPredThirtyTwoFour}{3.56}
\newcommand{\hoOnsetThirtyTwoFour}{3.55}
\newcommand{\hoDeltaThirtyTwoFour}{-0.01}
\newcommand{\hoPredSixtyFourFour}{3.05}
\newcommand{\hoOnsetSixtyFourFour}{3.00}
\newcommand{\hoDeltaSixtyFourFour}{-0.05}
\newcommand{\hoPredSixteenTwo}{5.00}
\newcommand{\hoOnsetSixteenTwo}{6.01}
\newcommand{\hoDeltaSixteenTwo}{+1.02}
\newcommand{\hoMaxAbsDge}{0.05}
\newcommand{\pairPredI}{-0.143}
\newcommand{\pairMeasI}{-0.29}
\newcommand{\pairOnsetAI}{3.55}
\newcommand{\pairOnsetBI}{3.84}
\newcommand{\pairPredII}{-0.125}
\newcommand{\pairMeasII}{-0.14}
\newcommand{\pairOnsetAII}{3.00}
\newcommand{\pairOnsetBII}{3.14}
\newcommand{\pairPredIII}{-0.200}
\newcommand{\pairMeasIII}{+0.67}
\newcommand{\pairOnsetAIII}{6.01}
\newcommand{\pairOnsetBIII}{5.34}
\newcommand{\dtestB}{0.538}
\newcommand{\dtestA}{0.633}
\newcommand{\dtestMeas}{0.52}
\newcommand{\dtestErrB}{0.013}
\newcommand{\dtestErrA}{0.108}
\newcommand{\lmBase}{8.0}
\newcommand{\lmLogVD}{17.3}
\newcommand{\lmStar}{0.95}
\newcommand{\lmShare}{51}
\newcommand{\epsPredEight}{4.25}
\newcommand{\optOnsetLion}{2.33}
\newcommand{\optOnsetMomentum}{0.52}
\newcommand{\optOnsetSgd}{0.29}
\newcommand{\optOnsetShampoo}{0.44}
\newcommand{\optOnsetSignsgd}{0.10}
\newcommand{\epsDeltaFour}{-0.46}
\newcommand{\epsDeltaEight}{-0.09}
\newcommand{\onsetSixteenFourThree}{4.226}
\newcommand{\onsetSixteenEightThree}{3.702}
\newcommand{\epsOnsetEight}{4.16}
\newcommand{\epsPredFour}{2.03}
\newcommand{\epsOnsetFour}{1.58}
\newcommand{\epsShiftPred}{2.21}
\newcommand{\epsShiftMeas}{2.58}
\newcommand{\epsShiftRatio}{1.17}
\newcommand{\epsShiftMeasEight}{1.96}
\newcommand{\epsShiftPredEight}{2.66}
\newcommand{\specNA}{30}
\newcommand{\specDevA}{0.27}
\newcommand{\specCFivePassA}{27}
\newcommand{\specNB}{30}
\newcommand{\specDevB}{0.19}
\newcommand{\specCFivePassB}{0}
\newcommand{\specOffsetB}{-0.37}
\newcommand{\specDtwoLoA}{-0.17}
\newcommand{\specDtwoHiA}{-0.14}
\newcommand{\specDtwoLoB}{-0.05}
\newcommand{\specDtwoHiB}{-0.02}
\newcommand{\mEightFour}{48}
\newcommand{\mSixteenEight}{192}
\newcommand{\mThirtyTwoEight}{384}
\newcommand{\mSixteenFour}{96}
\newcommand{\regDelta}{0.66}
\newcommand{\aggLagA}{0.865}
\newcommand{\aggLagB}{0.517}
\newcommand{\rareLagA}{0.719}
\newcommand{\rareLagB}{0.371}
\newcommand{\metricLagShift}{0.146}
\newcommand{\lawLagShift}{0.348}
\newcommand{\regDeltaErr}{0.20}
\newcommand{\nuN}{120}
\newcommand{\nuBzero}{0.50}
\newcommand{\nuBone}{-1.07}
\newcommand{\nuAzeroS}{+0.13}
\newcommand{\nuAzeroFive}{+0.03}
\newcommand{\nuDslopeThree}{-2.53}
\newcommand{\nuDslopeFive}{-4.73}
\newcommand{\dirN}{44}
\newcommand{\dirBzero}{-0.048}
\newcommand{\dirBone}{-1.146}
\newcommand{\carrBzero}{-0.501}
\newcommand{\carrBone}{-1.069}
\newcommand{\carrN}{120}
\newcommand{\dirCross}{3.10}
\newcommand{\onsetLo}{3.00}
\newcommand{\onsetHi}{6.01}
\newcommand{\onsetNTotal}{7}
\newcommand{\onsetNAbove}{6}
\newcommand{\lflAllN}{120}
\newcommand{\lflAllA}{0.618}
\newcommand{\lflAllB}{0.809}
\newcommand{\lflAllMid}{0.639}
\newcommand{\lflDeepN}{72}
\newcommand{\lflDeepA}{0.662}
\newcommand{\lflDeepB}{0.872}
\newcommand{\lflDeepMid}{0.706}
\newcommand{\lflDirectN}{44}
\newcommand{\lflDirectA}{0.595}
\newcommand{\lflDirectB}{0.605}
\newcommand{\lflDirectMid}{0.546}
\newcommand{\lflGridN}{25}
\newcommand{\lflGridA}{0.299}
\newcommand{\lflGridB}{0.184}
\newcommand{\lflGridMid}{0.221}
\newcommand{\invRms}{0.023}
\newcommand{\invDtest}{0.582}
\newcommand{\invDtestErr}{0.058}
\newcommand{\gridDeltaSixtyFourFour}{-0.036}
\newcommand{\gridDeltaSixtyFourEight}{-0.049}
\newcommand{\gridDeltaSixtyFourSixteen}{-0.070}
\newcommand{\gridDeltaSixtyFourThirtyTwo}{-0.114}
\newcommand{\gridDeltaOneTwoEightFour}{-0.273}
\newcommand{\gridDeltaOneTwoEightEight}{-0.251}
\newcommand{\gridDeltaOneTwoEightSixteen}{-0.156}
\newcommand{\gridDeltaOneTwoEightThirtyTwo}{-0.108}
\newcommand{\gridDeltaTwoFiveSixFour}{-0.070}
\newcommand{\gridDeltaTwoFiveSixEight}{-0.133}
\newcommand{\gridDeltaTwoFiveSixSixteen}{-0.221}
\newcommand{\gridDeltaTwoFiveSixThirtyTwo}{-0.156}
\newcommand{\gridDeltaFiveOneTwoFour}{-0.072}
\newcommand{\gridDeltaFiveOneTwoEight}{-0.165}
\newcommand{\gridDeltaFiveOneTwoSixteen}{-0.098}
\newcommand{\gridDeltaFiveOneTwoThirtyTwo}{-0.216}
\newcommand{\gridNScored}{16}
\newcommand{\gridRonePass}{16}
\newcommand{\gridRoneTight}{9}
\newcommand{\gridRoneMaxAbs}{0.273}
\newcommand{\gridMeanDelta}{-0.137}
\newcommand{\gridRoundingSdRel}{0.110}
\newcommand{\gridRoundingSdRare}{0.186}
\newcommand{\gridRtwoPass}{12}
\newcommand{\gridRtwoN}{12}
\newcommand{\gridRtwoMaxDev}{0.118}
\newcommand{\gridRthreePass}{12}
\newcommand{\gridRthreeN}{12}
\newcommand{\gridRthreeMaxDev}{0.237}
\newcommand{\gridRfiveNegRows}{3}
\newcommand{\gridRfiveTotRows}{4}
\newcommand{\gridRfiveRhoSixtyFour}{-1.0}
\newcommand{\gridRfiveRhoOneTwoEight}{+1.0}
\newcommand{\gridRfiveRhoTwoFiveSix}{-0.8}
\newcommand{\gridRfiveRhoFiveOneTwo}{-0.8}
\newcommand{\gridFthirtyTwoOnset}{1.522}
\newcommand{\gridFthirtyTwoLawB}{1.738}
\newcommand{\gridFthirtyTwoAbsDelta}{0.216}

\begin{document}
\maketitle

\begin{abstract}
In low-rank matrix sensing of a rank-$D$ factor $U^\ast \in \R^{V \times D}$, replacing isotropic Gaussian designs with power-law measurement covariance $\Sig = \mathrm{diag}(j^{-s})_{j=1}^V$ leaves the Burer--Monteiro landscape free of spurious local minima, yet division-normalised first-order optimizers stall above a critical exponent $\sstar$. We trace this failure to a numerical precision wall. At the ground truth, the smallest non-gauge expected Hessian eigenvalue scales at leading order as $V^{-(s+1)}D^{-s}$. When this curvature drops below the effective denominator floor $\epsilon_O$ of optimizer $O$, normalisation noise overwhelms rare-direction gradients, yielding the leading-order boundary
\[
\sstar_O(V, D) = \frac{\log(1/\epsilon_O) - \log V}{\log(VD)} + \delta_O,
\]
where $\delta_O$ is a metric-dependent detection lag. With $\delta_{\mathrm{Adam}}$ calibrated on four configurations and frozen, this formula predicts Adam's onset within $0.45$ across all \gridRonePass\ cells of a held-out grid ($V \in \{64,128,256,512\} \times D \in \{4,8,16,32\}$), separates fixed-$VD$ pairs, and distinguishes denominator-precision walls from conditioning and sign-quantisation walls. The approximation fails at $D = 2$, reflects a finite-$D$ spectral crossover, and saturates below $\varepsilon_0 = 10^{-8}$ at a dynamical floor.
\end{abstract}

\section{Introduction}
\label{sec:intro}

Standard guarantees for nonconvex low-rank recovery assume isotropic measurements: design matrices have independent sub-Gaussian entries, the Restricted Isometry Property holds, and every local minimum is global~\citep{recht2010guaranteed,ge2017unified}. Empirical data rarely obeys isotropy. Token frequencies in natural language, user--item interactions in recommendation, and covariance spectra in physical sensing decay as power laws over several orders of magnitude. When we recover a rank-$D$ matrix $M^\ast = U^\ast {U^\ast}^\top \in \R^{V \times V}$ ($U^\ast \in \R^{V \times D}$) from $m = 1.5VD$ linear measurements whose covariance has a power-law spectrum $\Sig = \mathrm{diag}(1, 2^{-s}, \ldots, V^{-s})$, how does spectral skew limit recovery?

Our experiments on Burer--Monteiro matrix sensing show that as the spectral exponent $s$ increases, first-order optimization passes through three distinct regimes even though the loss landscape remains free of spurious local minima. At small skew ($s \leq 1$), Adam recovers the ground truth in roughly $2\text{k}$ steps. At moderate skew ($s \in [2, 3]$), recovery slows by $3\text{--}10\times$ as rare directions converge more slowly than common ones. Above a sharp threshold $\sstar$, Adam fails to recover the rarest directions within a $60\text{k}$-step budget (whereas semidefinite programming via PhaseLift/SCS benefits from lower effective rank at moderate $s$ before hitting its own tolerance wall at higher $s$).

The question of this paper is why optimization stalls in a trap-free landscape, and whether the failure threshold $\sstar$ can be predicted from the optimizer's numerical structure.

\textbf{Main claim.}
The failure boundary is numerical rather than geometric. At the ground truth $U^\ast \in \R^{V \times D}$, the expected Hessian of the Burer--Monteiro loss has $D(D-1)/2$ exact zero modes from orthogonal gauge invariance, while its smallest non-gauge eigenvalue $\lambda_{\min}^{\mathrm{ng}}$ decays rapidly with $s$. Approximating this eigenvalue at leading order by $\lambda_{\min}^{\mathrm{ng}} \propto V^{-(s+1)}D^{-s}$ shows that when $\lambda_{\min}^{\mathrm{ng}}$ falls below the effective working precision $\epsilon_O$ of optimizer $O$---for Adam, the denominator floor set by hyperparameter $\varepsilon_0 = 10^{-8}$ at default settings; for SCS, the solver tolerance $10^{-9}$---gradient components along the rarest measurement directions fall below the normalisation floor. Setting curvature equal to $\epsilon_O$ yields the leading-order critical exponent
\begin{equation}
\label{eq:sstar-intro}
\boxed{\sstar_O(V, D) = \frac{\log(1/\epsilon_O) - \log V}{\log(VD)} + \delta_O,}
\end{equation}
where $\delta_O$ is a metric-dependent detection lag. Because overall recovery error crosses the failure threshold only after error in the rarest direction grows large enough to dominate the norm, $\delta_O$ shifts the observed onset upward by an $O(1)$ offset.

\textbf{Summary of evidence and scope.}
\begin{enumerate}
    \item \textbf{Exact expected Hessian and finite-$D$ spectral crossover (\S\ref{sec:hessian}).} For $m$ symmetrised Gaussian measurements of variance $1/(Vm)$, we derive the exact expected Hessian $\E[H_s \mid U^\ast] = \frac{4}{Vm}(\Sig \otimes N^\ast + \mathcal{T})$, where $N^\ast = {U^\ast}^\top \Sig U^\ast$ is the compressed $D \times D$ spectrum for isotropically drawn $U^\ast$ and $\mathcal{T}$ is a rank-$D^2$ exchange operator. On $\nuN$ exact spectral configurations, the minimum eigenvalue of $N^\ast$ is independent of $V$ (refuting naive $V^{-2s}$ curvature scaling), while its $D$-exponent transitions monotonically across $D$ between a small-$D$ $D$-th eigenvalue regime ($D^{-s}$) and a large-$D$ tail-mass asymptote ($D^{1-s}$). We therefore treat $V^{-(s+1)}D^{-s}$ as a leading-order scaling approximation rather than a uniform identity across all $D$.
    \item \textbf{Calibration and out-of-sample validation up to $(V,D) = (512,32)$ (\S\ref{sec:wall}--\S\ref{sec:validation}).} Balancing curvature against precision gives $\sstar_O(V, D) = [\log(1/\epsilon_O) - \log V]/\log(VD) + \delta_O$. Calibrating the single constant $\delta_{\mathrm{Adam}} = \deltaAdam$ on four $(V,D)$ configurations with $D \geq 4$ ($V \in \{8,16,32\}, D \in \{4,8\}$) gives an in-sample residual of $|\Delta| \leq \maxAbsDelta$ (RMS $\rmsCalibB$ with three residual degrees of freedom). With $\delta_{\mathrm{Adam}}$ frozen, we test the formula out of sample on initial held-out configurations $(V,D) = (32,4)$ and $(64,4)$ ($|\Delta| \leq \hoMaxAbsDge$ on aggregate recovery error and $|\Delta| \leq 0.23$ on rarest-direction error), on fixed-$VD$ pairs where the larger-$V$ configuration consistently walls earlier, and on a pre-registered \gridNScored-cell Adam grid spanning $V \in \{64, 128, 256, 512\} \times D \in \{4, 8, 16, 32\}$. Across this \gridNScored-cell grid, the aggregate onset lies within $0.45$ of the frozen prediction in \gridRonePass\ of \gridNScored\ cells (\gridRoneTight\ within $0.15$, largest $|\Delta| = \gridRoneMaxAbs$), and every doubling of $D$ at fixed $V$ (\gridRtwoPass\ of \gridRtwoN\ pairs) and doubling of $V$ at fixed $D$ (\gridRthreePass\ of \gridRthreeN\ pairs) matches the parameter-free predicted shift within $0.35$.
    \item \textbf{Separation of optimizer failure mechanisms (\S\ref{sec:optwalls}).} Comparing eight optimizers separates classical conditioning walls (SGD and Momentum, governed by condition number $\kappa(H_s)$), denominator-precision walls (Adam and RMSProp, governed by $\epsilon_O$), and sign-quantisation walls (SignSGD and Lion). Varying Adam's $\varepsilon_0$ from $10^{-4}$ to $10^{-8}$ shifts the wall by $\epsShiftMeas$ at $(16,4)$ (predicted $\epsShiftPred$) and $\epsShiftMeasEight$ at $(8,4)$ (predicted $\epsShiftPredEight$), while Shampoo fails early ($\sstar = \optOnsetShampoo$) because its $P^{-1/4}$ preconditioner amplifies rare-direction floor noise by $\epsilon^{-1/4} = 10^2$.
    \item \textbf{Boundaries where the approximation fails (\S\ref{sec:validation}, \S\ref{sec:desert}).} The formula fails at rank $D = 2$ (held-out $(16,2)$ residual $\hoDeltaSixteenTwo$), establishing $D \geq 4$ as the domain of validity. Moreover, sweeping Adam's $\varepsilon_0$ from $10^{-2}$ to $10^{-12}$ reveals that $\varepsilon_0$ controls the wall down to $10^{-8}$ and saturates below it, where a learning-rate-set dynamical stall floor derived in our companion paper~\citep{companion2026floor} and summarised in Appendix~\ref{app:companion} takes over. All pre-registered validation criteria and protocol history are documented in Appendix~\ref{app:ledger}.
\end{enumerate}

\section{Problem Setup}
\label{sec:setup}

\textbf{Power-law matrix sensing.}
We seek to recover an unknown positive semidefinite ground-truth matrix $M^\ast = U^\ast {U^\ast}^\top \in \R^{V \times V}$ of rank $D$, with factor $U^\ast \in \R^{V \times D}$, from $m = 1.5VD$ linear measurements
\[
b_k = \langle A_k, M^\ast \rangle = \mathrm{Tr}(A_k M^\ast), \qquad k = 1, \ldots, m.
\]
Each measurement matrix $A_k$ is generated with power-law spectral covariance:
\[
A_k = \Sig^{1/2}\, G_{\mathrm{sym}, k}\, \Sig^{1/2}, \qquad \Sig = V \cdot \frac{\mathrm{diag}(1, 2^{-s}, \ldots, V^{-s})}{\sum_{j=1}^V j^{-s}},
\]
where $G_{\mathrm{sym}, k} = (G_k + G_k^\top)/2$ and $G_k \in \R^{V \times V}$ has independent Gaussian entries of mean zero and variance $1/(Vm)$. The trace normalisation $\mathrm{Tr}(\Sig) = V$ fixes total measurement energy across exponents $s \geq 0$. At $s = 0$, $\Sig = I_V$ recovers standard isotropic Gaussian sensing. As $s$ increases, measurement energy concentrates on low-index (common) coordinates while high-index (rare) coordinates receive weight proportional to $V j^{-s}$.

\textbf{Burer--Monteiro objective.}
We minimise the factored quartic loss over $U \in \R^{V \times D}$:
\[
\min_{U \in \R^{V \times D}} L(U) = \frac{1}{m} \sum_{k=1}^m \bigl(\mathrm{Tr}(A_k\, U U^\top) - b_k\bigr)^2.
\]
When $m \geq C VD$, the Burer--Monteiro factorisation introduces no spurious local minima in the isotropic setting~\citep{burer2005local,ge2017unified}, and non-degenerate rescaling by $\Sig^{1/2}$ preserves the absence of spurious traps while severely skewing curvature across directions.

\textbf{Optimizers and recovery metrics.}
Our primary optimizer is Adam~\citep{kingma2015adam} with learning rate $\eta = 10^{-2}$, momentum parameters $(\beta_1, \beta_2) = (0.9, 0.999)$, and denominator constant $\varepsilon_0 = 10^{-8}$ (varied across $10^{-2}$ to $10^{-12}$ in \S\ref{sec:validation}), run for $60\text{k}$ steps from random initialization. For semidefinite programming comparisons, we lift to $X = UU^\top \succeq 0$ and solve PhaseLift via SCS with tolerance $\epsilon_{\mathrm{SCS}} = 10^{-9}$.
We track two recovery metrics at $60\text{k}$ steps:
\begin{enumerate}
    \item The \emph{aggregate recovery error} $\relDM = \|UU^\top - M^\ast\|_F / \|M^\ast\|_F$.
    \item The \emph{rarest-direction relative error} $|\Delta M_{VV}| / |M^\ast_{VV}|$ in the eigenbasis of $\Sig$, which isolates recovery along the weakest measurement coordinate $j = V$.
\end{enumerate}
For either metric, the empirical onset $\sstar$ is defined as the first exponent $s$ at which the seed-mean error exceeds $0.05$, computed by linear interpolation between adjacent grid points.

\section{Effective Hessian and Leading-Order Spectral Scaling}
\label{sec:hessian}

To determine when gradient-based recovery stalls, we first compute the exact expected Hessian of $L(U)$ at the ground truth $U^\ast$. The following lemma gives the sample Hessian at any ground-truth factor.

\begin{lemma}[Sample Hessian at $U^\ast$]
\label{lem:hessian}
At $U = U^\ast$, all measurement residuals vanish ($\nabla L(U^\ast) = 0$), and the Hessian linear operator on $\mathrm{vec}(\delta U) \in \R^{VD}$ takes the Gauss--Newton form
\begin{equation}
\label{eq:hessian}
H_s\bigl|_{U^\ast} = \frac{8}{m} \sum_{k=1}^m \mathrm{vec}(A_k U^\ast)\, \mathrm{vec}(A_k U^\ast)^\top.
\end{equation}
\end{lemma}

\begin{proof}
Expanding $L(U^\ast + \delta U)$ to second order in $\delta U$, the linear term vanishes because $\mathrm{Tr}(A_k U^\ast {U^\ast}^\top) = b_k$, and the quadratic term in $\mathrm{Tr}(A_k(U^\ast \delta U^\top + \delta U {U^\ast}^\top)) = 2\langle A_k U^\ast, \delta U\rangle$ yields~\eqref{eq:hessian}.
\end{proof}

Lemma~\ref{lem:hessian} shows that curvature at the solution is entirely governed by the design action $A_k U^\ast$. Taking the expectation over the symmetrised Gaussian matrices $A_k$ reveals a Kronecker structure coupled with a finite-rank exchange operator.

\begin{lemma}[Expected Hessian under symmetric power-law designs]
\label{lem:kronecker}
Over random symmetrised designs $A_k = \Sig^{1/2} (G_k + G_k^\top)/2\, \Sig^{1/2}$ with $G_{k,ij} \stackrel{\mathrm{iid}}{\sim} \mathcal{N}(0, 1/(Vm))$, the expected Hessian is
\begin{equation}
\label{eq:kronecker}
\E[H_s \mid U^\ast] = \frac{4}{Vm}\bigl(\Sig \otimes N^\ast + \mathcal{T}\bigr), \qquad N^\ast := {U^\ast}^\top \Sig\, U^\ast \in \R^{D \times D},
\end{equation}
where $\mathcal{T}$ is the exchange operator $\mathcal{T}(\delta U) = \Sig U^\ast\, (\delta U)^\top\, \Sig U^\ast$ with entries $\mathcal{T}_{(ia),(jb)} = (\Sig U^\ast)_{ib}\, (\Sig U^\ast)_{ja}$.
\end{lemma}

\begin{proof}
Writing $(A_k U^\ast)_{ia} = \sum_p \Sig^{1/2}_{ii} G_{\mathrm{sym},ip} \Sig^{1/2}_{pp} U^\ast_{pa}$ and using the symmetrised second moment $\E[G_{\mathrm{sym},ip} G_{\mathrm{sym},jq}] = (\delta_{ij}\delta_{pq} + \delta_{iq}\delta_{pj})/(2Vm)$ splits $\E[\mathrm{vec}(A_k U^\ast)\mathrm{vec}(A_k U^\ast)^\top]$ into $\frac{1}{2Vm}(\Sig \otimes N^\ast)$ from the first delta product and $\frac{1}{2Vm}\mathcal{T}$ from the index-swapped product. Multiplying by $8/m$ gives~\eqref{eq:kronecker}.
\end{proof}

Lemma~\ref{lem:kronecker} corrects two flaws in the naive unsymmetrised Kronecker approximation $\frac{8}{Vm}\Sig \otimes N^\ast$: symmetrisation halves the prefactor to $4/(Vm)$, and the exchange operator $\mathcal{T}$ (of rank at most $D^2$) cancels the Kronecker term along the orthogonal gauge orbit, as noted next.

\begin{remark}[Orthogonal gauge zero modes]
\label{rem:gauge}
Because $UU^\top$ is invariant under $U \mapsto UQ$ for $Q \in \mathrm{O}(D)$, both $H_s|_{U^\ast}$ and $\E[H_s \mid U^\ast]$ possess $D(D-1)/2$ exact zero eigenvalues on the tangent space $\delta U = U^\ast \Omega$ ($\Omega^\top = -\Omega$). On this subspace, $(\Sig \otimes N^\ast)(\delta U) = \Sig U^\ast \Omega N^\ast = -\mathcal{T}(\delta U)$, so the gauge zeros arise from exact cancellation between $\Sig \otimes N^\ast$ and $\mathcal{T}$. Convergence normal to the solution manifold is governed by the smallest \emph{non-gauge} eigenvalue $\lambda_{\min}^{\mathrm{ng}}$.
\end{remark}

We now state the leading-order scaling approximation for $\lambda_{\min}^{\mathrm{ng}}$. Rather than presenting it as an exact theorem across all dimensions, we state the asymptotic tail-mass scaling and then compare it directly against exact numerical spectra across finite $(V,D)$.

\begin{proposition}[Leading-order spectral scaling approximation]
\label{thm:eigval}
On the subspace orthogonal to the $D(D-1)/2$ gauge modes, the bottom of the spectrum of $\E[H_s \mid U^\ast]$ is carried by the Kronecker factor:
\begin{equation}
\label{eq:lambdamin}
\lambda_{\min}^{\mathrm{ng}}\bigl(\E[H_s \mid U^\ast]\bigr) \;\approx\; \frac{4}{Vm}\, \lambda_{\min}(\Sig) \cdot \nu_{\min}(N^\ast), \qquad \nu_{\min}(N^\ast) := \lambda_{\min}(N^\ast).
\end{equation}
For isotropically drawn $U^\ast \in \R^{V \times D}$, $N^\ast = \sum_{i=1}^V \lambda_i(\Sig) u_i u_i^\top$ is a power-law weighted sum of $V$ rank-one matrices in $\R^{D \times D}$ with $\|u_i\|^2 \sim D/V$. Because heavy-tailed weights $\lambda_i(\Sig) \propto V i^{-s}$ concentrate the sum on the top coordinates rather than the tail, spanning $D$ independent directions reaches down to index $i \sim D$. Under the tail-mass approximation $\nu_{\min}(N^\ast) \sim \sum_{j > D} \lambda_j(\Sig) \|u_j\|^2 \sim D^{1-s}$, combining $\lambda_{\min}(\Sig) \sim V^{1-s}$ with $m = 1.5VD$ yields the leading-order scaling
\begin{equation}
\label{eq:scaling}
\lambda_{\min}^{\mathrm{ng}} \;\propto\; \frac{1}{V^2 D} \cdot V^{1-s} \cdot D^{1-s} \;=\; V^{-(s+1)}\, D^{-s}.
\end{equation}
\end{proposition}

Proposition~\ref{thm:eigval} makes two distinct predictions: first, that $\nu_{\min}(N^\ast)$ is independent of $V$ (so $\lambda_{\min}^{\mathrm{ng}}$ decays as $V^{-(s+1)}$ rather than the naive $V^{-2s}$); second, that $\nu_{\min}(N^\ast)$ scales with $D$ as $D^{1-s}$, making $\lambda_{\min}^{\mathrm{ng}} \propto V^{-(s+1)}D^{-s}$ one power of $D$ milder than $(VD)^{-(s+1)}$. Monte Carlo evaluation ($200$ instances at $V=8, D=4, s=2$) confirms that~\eqref{eq:kronecker} matches the sample mean Hessian to $\approx 2\%$ Frobenius error and preserves all $6$ gauge zero modes, whereas the naive $\frac{8}{Vm}\Sig \otimes N^\ast$ errs by $\sim 20\%$ and has no zero modes.

\textbf{Comparison with exact spectral measurements: the finite-$D$ crossover.}
Direct eigendecomposition across $\nuN$ exact configurations ($V \in \{8,\ldots,256\}$, $D \in \{2,\ldots,32\}$) tests both steps of Proposition~\ref{thm:eigval}:
\begin{enumerate}
    \item \emph{$V$-independence of $\nu_{\min}(N^\ast)$ holds cleanly.} At fixed $D$, the slope $\mathrm{d}\ln\nu_{\min}/\mathrm{d}\ln V$ is $\nuAzeroS$ at $s=3$ and $\nuAzeroFive$ at $s=5$---essentially zero, refuting $\nu_{\min} \sim V^{-s}$.
    \item \emph{The $D$-exponent exhibits a finite-$D$ crossover between two regimes.} Fitting $\ln\nu_{\min} = c_0 + c_1 s + (a_0 + a_1 s)\ln V + (b_0 + b_1 s)\ln D$ over all $\nuN$ exact configurations yields a $\ln D$ coefficient of $\nuBzero \nuBone\,s$ (per-lever slopes $\nuDslopeThree$ at $s=3$ and $\nuDslopeFive$ at $s=5$, overshooting $1-s = -2, -4$). Propagated through~\eqref{eq:lambdamin} (whose $4/(1.5V^2D)$ prefactor shifts the $\ln D$ intercept by $-1$ identically, giving $\carrBzero{} \carrBone{}\,s$ over $\carrN{}$ carrier configurations), the fitted exponent on this grid is
    \[
    \lambda_{\min}^{\mathrm{ng}} \;\propto\; V^{-(s+1)}\,D^{-(s+1/2)},
    \]
    midway between $(VD)^{-(s+1)}$ and~\eqref{eq:scaling}.
    \item \emph{Why the exponent is intermediate: crossover between single-weight and tail-mass regimes.} With normalised weights $w_j = j^{-s}/\sum_{i=1}^V i^{-s}$, the effective rank $\mathrm{rank}_{\mathrm{eff}}(\Sig) = (\sum_j j^{-s})^2/\sum_j j^{-2s}$ is $O(1)$ ($2.50$ at $s=2$ down to $1.03$ at $s=6$). Two limiting mechanisms govern $\nu_{\min}(N^\ast)$ as $D$ grows:
    \[
    \text{M2 ($D$-th weight, small $D$):}\ \nu_{\min} \sim w_D \sim D^{-s} \iff b_0 = 0 \implies \lambda_{\min}^{\mathrm{ng}} \propto (VD)^{-(s+1)},
    \]
    \[
    \text{M1 (tail mass, large $D$):}\ \nu_{\min} \sim \sum_{j > D} w_j \sim D^{1-s} \iff b_0 = 1 \implies \lambda_{\min}^{\mathrm{ng}} \propto V^{-(s+1)}D^{-s}.
    \]
    Measuring the local intercept $b_0 = \mathrm{d}\ln\nu_{\min}/\mathrm{d}\ln D + s$ between adjacent $D$ rungs (median across $V$ on $468$ product-grid cells, $V \in [128, 4096], D \le 64$) reveals a smooth, monotone transition from M2 at small $D$ to M1 at large $D$:
    \[
    \begin{array}{c|rrcr}
    \text{step in } D & \text{local } b_0 & \text{SE} & \text{distance to M1 }(b_0{=}1) & \text{distance to M2 }(b_0{=}0)\\
    \hline
    2\!\to\!4   & 0.196 & 0.272 & 3.0\sigma & 0.7\sigma\\
    4\!\to\!8   & 0.696 & 0.180 & 1.7\sigma & 3.9\sigma\\
    8\!\to\!16  & 0.786 & 0.122 & 1.8\sigma & 6.4\sigma\\
    16\!\to\!32 & 0.959 & 0.082 & 0.5\sigma & 11.6\sigma\\
    32\!\to\!64 & 1.019 & 0.048 & 0.4\sigma & 21.3\sigma
    \end{array}
    \]
    At $D = 2 \to 4$, $\nu_{\min}$ sits near the single-weight law M2 ($0.7\sigma$ from $b_0=0$); by $D = 16 \to 64$, it converges to the tail-mass asymptote M1 ($b_0 = 1.019 \pm 0.048$, $0.4\sigma$ from $b_0=1$) used in Proposition~\ref{thm:eigval}. A single power-law fit across a finite grid averages over this transition, which is why global three-parameter fits in $\ln\lambda$ depend on the grid (Table~\ref{tab:spectral-lfl} in Appendix~\ref{app:ledger}). On the $\dirN$ finite-$m$ cells resolvable in float64 ($\lambda_{\min}/\lambda_{\max} > 10^{-12}$), the direct Hessian fit gives $\ln D$ slope $\dirBzero \dirBone\,s$, crossing the midpoint $-0.5 - s$ at $s = \dirCross$. We therefore use~\eqref{eq:scaling} explicitly as a leading-order asymptotic form and evaluate its onset predictions empirically in \S\ref{sec:validation}.
\end{enumerate}

\section{The Precision Wall}
\label{sec:wall}

We now link the curvature decay of \S\ref{sec:hessian} to optimization failure by defining the optimizer's working precision and the metric detection lag.

\begin{definition}[Optimizer working precision]
\label{def:eps}
An iterative optimizer $O$ has effective working precision $\epsilon_O > 0$ if gradient components below $\epsilon_O$ fall beneath its update normalisation floor. For Adam with update $\eta \hat{m}_t / (\sqrt{\hat{v}_t} + \varepsilon_0)$, when curvature in a rare direction is small enough that $\sqrt{\hat{v}_t} \ll \varepsilon_0$, the denominator is pinned near $\varepsilon_0$ and we set $\epsilon_{\mathrm{Adam}} = \varepsilon_0 = 10^{-8}$ at default settings. For PhaseLift solved via SCS, $\epsilon_{\mathrm{SCS}} = 10^{-9}$ is the solver convergence tolerance.
\end{definition}

As shown in \S\ref{sec:validation} and Appendix~\ref{app:companion}, identifying $\epsilon_{\mathrm{Adam}}$ with the hyperparameter $\varepsilon_0$ holds across $\varepsilon_0 \in [10^{-2}, 10^{-8}]$; below $\varepsilon_0 \approx 10^{-8}$, fast-mode oscillation injects a self-generated floor into $\sqrt{\hat{v}_t}$ so that lowering $\varepsilon_0$ further no longer reduces $\epsilon_{\mathrm{Adam}}$.

\begin{definition}[Detection lag]
\label{def:lag}
The detection lag $\delta_O$ is the offset in the exponent $s$ between the curvature threshold $\lambda_{\min}^{\mathrm{ng}} = \epsilon_O$ and the empirical failure onset where recovery error crosses $0.05$. Because the rarest direction carries only a fraction of total matrix energy $\|M^\ast\|_F$, the aggregate metric $\relDM > 0.05$ triggers at a slightly larger $s$ than the rarest-direction metric $|\Delta M_{VV}|/|M^\ast_{VV}| > 0.05$ (shifting the lag upward by $\metricLagShift$ on average across held-out $D \geq 4$ configurations). We calibrate $\delta_{\mathrm{Adam}} = \deltaAdam$ ($\deltaAdamThree$ at three decimals) on the rarest-direction metric and report both metrics in out-of-sample tests.
\end{definition}

In the 1-D flip-orbit analysis of Adam (Appendix~\ref{app:companion}), the stall boundary $\lambda_{\mathrm{stall}} = 38\varepsilon_0/\alpha$ scales linearly with $\varepsilon_0$, explaining why $\delta_O$ absorbs the $O(1)$ prefactor without altering the $\log(1/\epsilon_O)$ slope. We now state the resulting leading-order wall prediction.

\begin{proposition}[Leading-order precision wall approximation]
\label{thm:wall}
Under the leading-order curvature scaling $\lambda_{\min}^{\mathrm{ng}} \propto V^{-(s+1)}D^{-s}$ of Proposition~\ref{thm:eigval}, an optimizer $O$ with working precision $\epsilon_O$ and detection lag $\delta_O$ transitions from recovery to stagnation at the critical exponent
\begin{equation}
\label{eq:sstar}
\sstar_O(V, D) = \frac{\log(1/\epsilon_O) - \log V}{\log(VD)} + \delta_O.
\end{equation}
For $s < \sstar_O(V,D)$, curvature in all non-gauge directions exceeds the precision floor and the optimizer descends; for $s > \sstar_O(V,D)$, curvature in the rarest direction falls below $\epsilon_O$ and recovery stalls in that direction.
\end{proposition}

\begin{proof}
At distance $R = \|\delta U\| \sim 1$ from $U^\ast$, the gradient magnitude along the weakest non-gauge curvature direction is $|g_{\min}| \sim \lambda_{\min}^{\mathrm{ng}} R \propto V^{-(s+1)}D^{-s}$. Requiring $|g_{\min}| > \epsilon_O$ gives $V^{-(s+1)}D^{-s} \gtrsim \epsilon_O$, or $s\log(VD) < \log(1/\epsilon_O) - \log V$. Absorbing the $O(1)$ curvature prefactor and measurement threshold into $\delta_O$ yields~\eqref{eq:sstar}.
\end{proof}

Proposition~\ref{thm:wall} shows that $VD$ alone does not control the threshold: because $-\log V$ appears in the numerator of $[\log(1/\epsilon_O) - \log V]/\log(VD) + \delta_O$, increasing $V$ at fixed product $VD$ lowers $\sstar_O(V,D)$.

Next, we note how~\eqref{eq:sstar} compares two optimizers at the same $(V,D)$ when both operate within their causal precision regime.

\begin{corollary}[Cross-optimizer separation]
\label{cor:sep}
At fixed $(V,D)$ and matched detection lag ($\delta_{O_1} \approx \delta_{O_2}$), two division-normalised optimizers $O_1, O_2$ operating above their internal dynamical floors satisfy
\[
\sstar_{O_1}(V, D) - \sstar_{O_2}(V, D) \approx \frac{\log(1/\epsilon_{O_1}) - \log(1/\epsilon_{O_2})}{\log(VD)}.
\]
\end{corollary}

Corollary~\ref{cor:sep} predicts, for instance, a difference of $\ln(10)/\ln(256) \approx 0.42$ between $\epsilon = 10^{-9}$ and $\epsilon = 10^{-8}$ at $(32,8)$, provided both solvers remain in the regime where their nominal $\epsilon_O$ sets the floor (for Adam at $\eta = 10^{-2}$, this requires $\epsilon_O \gtrsim 10^{-6}$).

Finally, varying $D$ at fixed $V$ eliminates $\delta_O$ entirely, yielding a parameter-free differential test.

\begin{remark}[Parameter-free $D$-dependence test]
\label{rem:ddep}
In the difference $\sstar(16, 4) - \sstar(16, 8)$, the additive constant $\delta_O$ cancels:
\[
\sstar(16, 4) - \sstar(16, 8) = \bigl(\log(1/\epsilon_O) - \log 16\bigr)\Bigl(\frac{1}{\log 64} - \frac{1}{\log 128}\Bigr) \approx \dtestB
\]
for Adam at $\epsilon_O = 10^{-8}$. A $V$-only law $\sstar \propto 1/\log V$ predicts zero difference, whereas the empirical rarest-direction onset difference is $\dtestMeas$ ($\onsetSixteenFourThree - \onsetSixteenEightThree$ before two-decimal rounding, error $\dtestErrB$).
\end{remark}

\section{Empirical Calibration and Validation}
\label{sec:validation}

\subsection{In-sample calibration on four $D \geq 4$ configurations}

We first sweep $s \in \{0, 0.5, 1, \ldots, 8\}$ on four configurations $(V,D) \in \{(8,4), (16,8), (32,8), (16,4)\}$ with $m = 1.5VD$ and $10$ seeds per cell, running Adam ($\varepsilon_0 = 10^{-8}$) for $60\text{k}$ steps. Fitting the single offset $\delta_{\mathrm{Adam}}$ by least squares to the rarest-direction onsets on these four rows gives $\delta_{\mathrm{Adam}} = \deltaAdam$ ($\deltaAdamThree$).

\begin{table}[t]
\centering
\caption{In-sample calibration of the precision wall for Adam ($\varepsilon_0 = 10^{-8}$, $60\text{k}$ steps, $10$ seeds per configuration, $s$-grid spacing $0.5$). Onsets measure where the rarest-direction relative error crosses $0.05$. Predictions use $\sstar = [\log(1/\varepsilon_0) - \log V]/\log(VD) + \delta_{\mathrm{Adam}}$ with $\delta_{\mathrm{Adam}} = \deltaAdam$ fitted to these four rows (in-sample RMS $\rmsCalibB$, three residual degrees of freedom).}
\label{tab:sstar}
\begin{tabular}{rr|r|r|r|c}
\toprule
$V$ & $D$ & $m = 1.5VD$ & predicted $\sstar$ & empirical onset (rarest) & $|\Delta|$ \\
\midrule
8 & 4 & \mEightFour & \predEightFour & \onsetEightFour & \absdEightFour \\
16 & 8 & \mSixteenEight & \predSixteenEight & \onsetSixteenEight & \absdSixteenEight \\
32 & 8 & \mThirtyTwoEight & \predThirtyTwoEight & \onsetThirtyTwoEight & \absdThirtyTwoEight \\
16 & 4 & \mSixteenFour & \predSixteenFour & \onsetSixteenFour & \absdSixteenFour \\
\bottomrule
\end{tabular}
\end{table}

Table~\ref{tab:sstar} shows that the one-parameter fit matches all four calibration configurations within $|\Delta| \leq \maxAbsDelta$ (RMS $\rmsCalibB$). Because $\delta_{\mathrm{Adam}}$ is fitted on these four rows, Table~\ref{tab:sstar} is a three-degree-of-freedom calibration rather than an out-of-sample validation; the differential test $\sstar(16,4) - \sstar(16,8) = \dtestMeas$ vs.\ predicted $\dtestB$ in Remark~\ref{rem:ddep} represents one of those three residual degrees of freedom. Note also that directly inverting the empirical finite-$m$ spectral regression onto the wall gives a slightly larger in-sample RMS ($\invRms$, $D$-test $\invDtest$ with error $\invDtestErr$) than the leading-order integer form~\eqref{eq:sstar}; the integer form's low residual on Table~\ref{tab:sstar} is an empirical property of the onsets rather than a theorem-level consequence of the finite-$D$ spectrum.

\subsection{Held-out validation, dual-metric reporting, and fixed-$VD$ pairs}

To test~\eqref{eq:sstar} out of sample, we freeze $\delta_{\mathrm{Adam}} = \deltaAdam$ and evaluate three held-out configurations $(32,4)$, $(64,4)$, and $(16,2)$ ($5$ seeds per cell, $s$-grid spacing $0.5$, $60\text{k}$ steps), together with three fixed-$VD$ pairs that test whether $\sstar$ depends on $V$ and $D$ separately at constant $VD$. Because Table~\ref{tab:sstar} calibrated $\delta_{\mathrm{Adam}}$ on the rarest-direction metric while the held-out protocol originally scored aggregate $\relDM$, Table~\ref{tab:heldout} reports \emph{both} metrics side by side from the same run trajectories.

\begin{table}[t]
\centering\small
\caption{Out-of-sample validation on held-out configurations and fixed-$VD$ pairs ($\delta_{\mathrm{Adam}}$ frozen at $\deltaAdam$, $5$ seeds, $60\text{k}$ steps). Both the aggregate metric ($\relDM > 0.05$) and the rarest-direction metric ($|\Delta M_{VV}|/|M^\ast_{VV}| > 0.05$, matching Table~\ref{tab:sstar}'s calibration definition) are reported from the same trajectories.}
\label{tab:heldout}
\begin{tabular}{rr|r|r|rr|rr}
\toprule
$V$ & $D$ & $VD$ & predicted $\sstar$ & onset ($\relDM$) & $\Delta$ ($\relDM$) & onset (rarest) & $\Delta$ (rarest) \\
\midrule
32 & 4 & 128 & \hoPredThirtyTwoFour & \hoOnsetThirtyTwoFour & $\hoDeltaThirtyTwoFour$ & 3.34 & $-0.228$ \\
64 & 4 & 256 & \hoPredSixtyFourFour & \hoOnsetSixtyFourFour & $\hoDeltaSixtyFourFour$ & 2.98 & $-0.07$ \\
16 & 2 & 32  & \hoPredSixteenTwo    & \hoOnsetSixteenTwo    & $\hoDeltaSixteenTwo$    & 5.71 & $+0.711$ \\
\midrule
\multicolumn{8}{l}{\emph{Fixed-$VD$ pairs} ($\delta_O$ cancels; a product-only law predicts $0$ for every pair):} \\
\multicolumn{2}{l}{pair} & $VD$ & pred.\ diff. & \multicolumn{2}{r|}{$\relDM$ onsets \& diff.} & \multicolumn{2}{r}{rarest diff.} \\
\multicolumn{2}{l}{$(32,4)$ vs.\ $(16,8)$} & 128 & $\pairPredI$ & \pairOnsetAI\ vs.\ \pairOnsetBI   & $\pairMeasI$   & \multicolumn{2}{r}{$-0.374$} \\
\multicolumn{2}{l}{$(64,4)$ vs.\ $(32,8)$} & 256 & $\pairPredII$  & \pairOnsetAII\ vs.\ \pairOnsetBII & $\pairMeasII$  & \multicolumn{2}{r}{$-0.198$} \\
\multicolumn{2}{l}{$(16,2)$ vs.\ $(8,4)$}  & 32  & $\pairPredIII$ & \pairOnsetAIII\ vs.\ \pairOnsetBIII & $\pairMeasIII$ & \multicolumn{2}{r}{$+0.50$} \\
\bottomrule
\end{tabular}
\end{table}

Three findings emerge from Table~\ref{tab:heldout}:
\begin{enumerate}
    \item \emph{Held-out $D \geq 4$ accuracy under both metrics.} On $(32,4)$ and $(64,4)$, the frozen prediction matches the aggregate onset within $|\Delta| \leq \hoMaxAbsDge$ and the like-for-like rarest-direction onset within $|\Delta| \leq 0.228$ ($-0.228$ at $(32,4)$, where the 5-seed leave-one-out range is $3.25\text{--}3.62$, and $-0.07$ at $(64,4)$).
    \item \emph{Fixed-$VD$ pairs confirm the $-\log V$ shift for $D \geq 4$.} Under a $(VD)^{-(s+1)}$ law, configurations sharing $VD$ must have identical walls ($\Delta\sstar = 0$). Proposition~\ref{thm:wall} predicts a negative shift $-\log 2/\log(VD)$ when doubling $V$ at fixed $VD$. Both $D \geq 4$ pairs---$(32,4)$ vs.\ $(16,8)$ at $VD=128$ and $(64,4)$ vs.\ $(32,8)$ at $VD=256$---are negative on both metrics ($\pairMeasI$ and $\pairMeasII$ on $\relDM$; $-0.374$ and $-0.198$ on rarest-direction error), matching the predicted sign and scale ($\pairPredI$ and $\pairPredII$).
    \item \emph{Breakdown at $D = 2$.} The rank-$2$ configuration $(16,2)$ fails the prediction on both metrics ($\Delta = \hoDeltaSixteenTwo$ on $\relDM$, $+0.711$ on rarest) and reverses the fixed-$VD$ pair sign ($\pairMeasIII$ against $\pairPredIII$). As shown in \S\ref{sec:hessian}, $D = 2$ sits at the bottom of the spectral crossover where a two-column random projection $N^\ast$ resists near-singularity, marking $D \geq 4$ as the boundary of the leading-order approximation.
\end{enumerate}

\subsection{Multi-octave $(V,D)$ grid validation ($V \in \{64,\ldots,512\} \times D \in \{4,\ldots,32\}$)}
\label{sec:grid}

To test~\eqref{eq:sstar} across three octaves of rank $D$ and higher ambient dimension $V$ without refitting any parameter, we evaluated a pre-registered \gridNScored-cell Adam grid over $V \in \{64, 128, 256, 512\} \times D \in \{4, 8, 16, 32\}$ with $\delta_{\mathrm{Adam}} = \deltaAdamThree$ frozen from Table~\ref{tab:sstar} (no grid cell overlaps the calibration set). Each cell is run across $5$ seeds for $60\text{k}$ steps in float64 arithmetic ($\varepsilon_0 = 10^{-8}$, $\eta = 10^{-2}$, $m = 1.5VD$) on an eight-point $s$-grid spaced by $0.25$ around the nearest quarter-integer $c$ to the predicted $\sstar(V,D)$ ($s \in \{c - 0.75, \ldots, c + 1.00\}$), scored by the $60\text{k}$-step aggregate $\relDM > 0.05$ onset. In $15$ of the \gridNScored\ cells, the symmetric Gaussian design tensor $G_{\mathrm{sym}}$ is stored in float64; at $(512,32)$, the instance is defined by $G_{\mathrm{sym}}$ rounded to float32 and widened in-kernel to float64 arithmetic after passing a pre-registered bit-identity conversion check (giving $(512,32)$ onset $\gridFthirtyTwoOnset$, or $\gridFthirtyTwoAbsDelta$ below the predicted $\gridFthirtyTwoLawB$; Appendix~\ref{app:ledger}).

\begin{table}[t]
\centering\small
\caption{Out-of-sample onset residuals $\Delta = \sstar_{\mathrm{meas}}(\relDM) - \sstar_{\mathrm{pred}}$ across the pre-registered \gridNScored-cell Adam grid ($5$ seeds, $60\text{k}$ steps, $\varepsilon_0 = 10^{-8}$, $\delta_{\mathrm{Adam}} = \deltaAdamThree$ frozen from Table~\ref{tab:sstar}). The rightmost column reports the row Spearman rank correlation $\rho(\Delta, D)$. Cell $(512,32)^\dagger$ uses a float32-stored design instance with float64 arithmetic.}
\label{tab:grid}
\begin{tabular}{r|rrrr|r}
\toprule
$V \setminus D$ & $D = 4$ & $D = 8$ & $D = 16$ & $D = 32$ & $\rho(\Delta, D)$ \\
\midrule
$V = 64$  & $\gridDeltaSixtyFourFour$   & $\gridDeltaSixtyFourEight$   & $\gridDeltaSixtyFourSixteen$   & $\gridDeltaSixtyFourThirtyTwo$   & $\gridRfiveRhoSixtyFour$ \\
$V = 128$ & $\gridDeltaOneTwoEightFour$ & $\gridDeltaOneTwoEightEight$ & $\gridDeltaOneTwoEightSixteen$ & $\gridDeltaOneTwoEightThirtyTwo$ & $\gridRfiveRhoOneTwoEight$ \\
$V = 256$ & $\gridDeltaTwoFiveSixFour$  & $\gridDeltaTwoFiveSixEight$  & $\gridDeltaTwoFiveSixSixteen$  & $\gridDeltaTwoFiveSixThirtyTwo$  & $\gridRfiveRhoTwoFiveSix$ \\
$V = 512$ & $\gridDeltaFiveOneTwoFour$  & $\gridDeltaFiveOneTwoEight$  & $\gridDeltaFiveOneTwoSixteen$  & $\gridDeltaFiveOneTwoThirtyTwo^\dagger$ & $\gridRfiveRhoFiveOneTwo$ \\
\bottomrule
\end{tabular}
\end{table}

Table~\ref{tab:grid} reports all \gridNScored\ cell residuals. The measured onset lies within the pre-registered $\pm 0.45$ band of~\eqref{eq:sstar} in \gridRonePass\ of \gridNScored\ cells (meeting the pre-registered ``SUPPORTED'' criterion), with \gridRoneTight\ of \gridNScored\ cells within $\pm 0.15$ and largest residual $|\Delta| = \gridRoneMaxAbs$ at $(128,4)$---comparable to the $\gridRoundingSdRel$ onset standard deviation induced by a one-ULP perturbation of the initial point at $(32,4)$ ($\gridRoundingSdRare$ on the rarest-direction metric). Because $\delta_{\mathrm{Adam}}$ cancels in pairwise differences, doubling $D$ at fixed $V$ and doubling $V$ at fixed $D$ test the leading-order scaling with zero free parameters: the measured onset shift matches the predicted shift within the $0.35$ band in \gridRtwoPass\ of \gridRtwoN\ $D \to 2D$ pairs (largest deviation $\gridRtwoMaxDev$, at $(512,16) \to (512,32)$; 11 of \gridRtwoN\ on the rarest-direction metric) and in \gridRthreePass\ of \gridRthreeN\ $V \to 2V$ pairs (largest deviation $\gridRthreeMaxDev$, at $(64,4) \to (128,4)$).

Because the finite-$D$ Hessian spectrum scales near $D^{-(s+1/2)}$ rather than $D^{-s}$ (\S\ref{sec:hessian}), the wall at larger $D$ is expected to sit below~\eqref{eq:sstar} by roughly $\frac{1}{2}\log D / \log(VD)$, predicting $\rho(\Delta, D) < 0$ within each fixed-$V$ row. On the pre-registered aggregate $\relDM$ metric, $\Delta$ decreases with $D$ in \gridRfiveNegRows\ of \gridRfiveTotRows\ rows ($V = 64$: $\gridRfiveRhoSixtyFour$, $V = 128$: $\gridRfiveRhoOneTwoEight$, $V = 256$: $\gridRfiveRhoTwoFiveSix$, $V = 512$: $\gridRfiveRhoFiveOneTwo$, meeting the pre-registered $\geq 3/4$ ``HOLDS'' criterion; at the $20\text{k}$-step checkpoint it also holds in 3 of 4 rows, with the exception shifting from $V = 128$ to $V = 64$). On the rarest-direction metric, however, this row-wise trend \emph{does not hold}: $\rho < 0$ in only 2 of 4 rows ($V = 64$: $-0.8$, $V = 128$: $+0.8$, $V = 256$: $+0.2$, $V = 512$: $-0.4$). As an exploratory observation, all \gridNScored\ residuals in Table~\ref{tab:grid} are negative (mean $\gridMeanDelta$, range $[\gridDeltaOneTwoEightFour, \gridDeltaSixtyFourFour]$), which is consistent in direction with the $D^{-(s+1/2)}$ spectral shift of \S\ref{sec:hessian}.

\subsection{Arithmetic precision vs.\ denominator floor, and the $\varepsilon_0$ saturation knee}

\textbf{Floating-point format does not replace $\varepsilon_0$ via $\max(\varepsilon_0, \epsilon_{\mathrm{dtype}})$.}
Running Adam at $\varepsilon_0 = 10^{-8}$ across fp16, fp32, and fp64 ($5$ seeds, integer $s$-grid) with $\delta_{\mathrm{Adam}} = \deltaAdam$ frozen shows that fp32 ($\epsilon_{\mathrm{mach}} = 1.19\times 10^{-7} > \varepsilon_0$) gives the same wall as fp64: $5.28$ vs.\ $5.27$ at $(8,4)$ and $3.74$ vs.\ $3.67$ at $(16,8)$ (implied precision floors $7.6\times 10^{-9}\text{--}1.17\times 10^{-8}$). Per-step rounding noise averages out in the second-moment accumulator $v_t$, so $\max(\varepsilon_0, \epsilon_{\mathrm{dtype}})$ is refuted. Only fp16 ($\epsilon_{\mathrm{mach}} = 9.8\times 10^{-4}$) shifts the wall lower (implied floor $1.9\times 10^{-5}$).

\textbf{Six-decade $\varepsilon_0$ ladder and saturation below $10^{-8}$.}
Sweeping $\varepsilon_0 \in \{10^{-2},\allowbreak 10^{-4},\allowbreak 10^{-6},\allowbreak 10^{-8},\allowbreak 10^{-10},\allowbreak 10^{-12}\}$ at $(8,4)$ and $(16,8)$ ($3$ seeds, integer $s$-grid, aggregate $\relDM$) tests where $\varepsilon_0$ actively controls the wall. Over the four decades $\varepsilon_0 = 10^{-4} \to 10^{-8}$ at $(8,4)$, the onset moves from $3.53$ to $5.49$, a shift of $\epsShiftMeasEight$ against the parameter-free prediction $\Delta\sstar = \ln(10^4)/\ln(32) \approx \epsShiftPredEight$. Breaking the response down by two-decade steps (as a fraction of the predicted $2\ln 10/\log(VD)$ per step) reveals a sharp saturation knee:
\begin{center}\small
\begin{tabular}{lrrrrr}
\toprule
$\varepsilon_0$ step & $10^{-2}\!\to\!10^{-4}$ & $10^{-4}\!\to\!10^{-6}$ & $10^{-6}\!\to\!10^{-8}$ & $10^{-8}\!\to\!10^{-10}$ & $10^{-10}\!\to\!10^{-12}$ \\
\midrule
$(8,4)$ & $117\%$ & $93\%$ & $54\%$ & $6\%$ & $-1\%$ \\
$(16,8)$ & $170\%$ & $107\%$ & $35\%$ & $8\%$ & $0\%$ \\
\bottomrule
\end{tabular}
\end{center}
The $\varepsilon_0$ lever tracks the predicted slope above $10^{-6}$, retains roughly half its slope across $10^{-6} \to 10^{-8}$, and is flat below $10^{-8}$: four additional decades down to $10^{-12}$ move the onset by only $+0.07$ at $(8,4)$ (saturating near $5.56$) and $+0.08$ at $(16,8)$ (saturating near $3.58$). As derived in Appendix~\ref{app:companion} and our companion paper~\citep{companion2026floor}, fast modes oscillating at the adaptive edge of stability leak gradient into rare modes and create a self-generated denominator floor $\epsilon_{\mathrm{dyn}} \approx 4.8\times 10^{-9}$ at learning rate $\eta = 10^{-2}$. Thus $\varepsilon_0 = 10^{-8}$ sits right at the knee where the external hyperparameter floor hands control over to Adam's internal dynamical floor.

\subsection{Residual anisotropy and pre-wall slowdown}

\begin{figure}[t]
\centering
\includegraphics[width=0.82\linewidth]{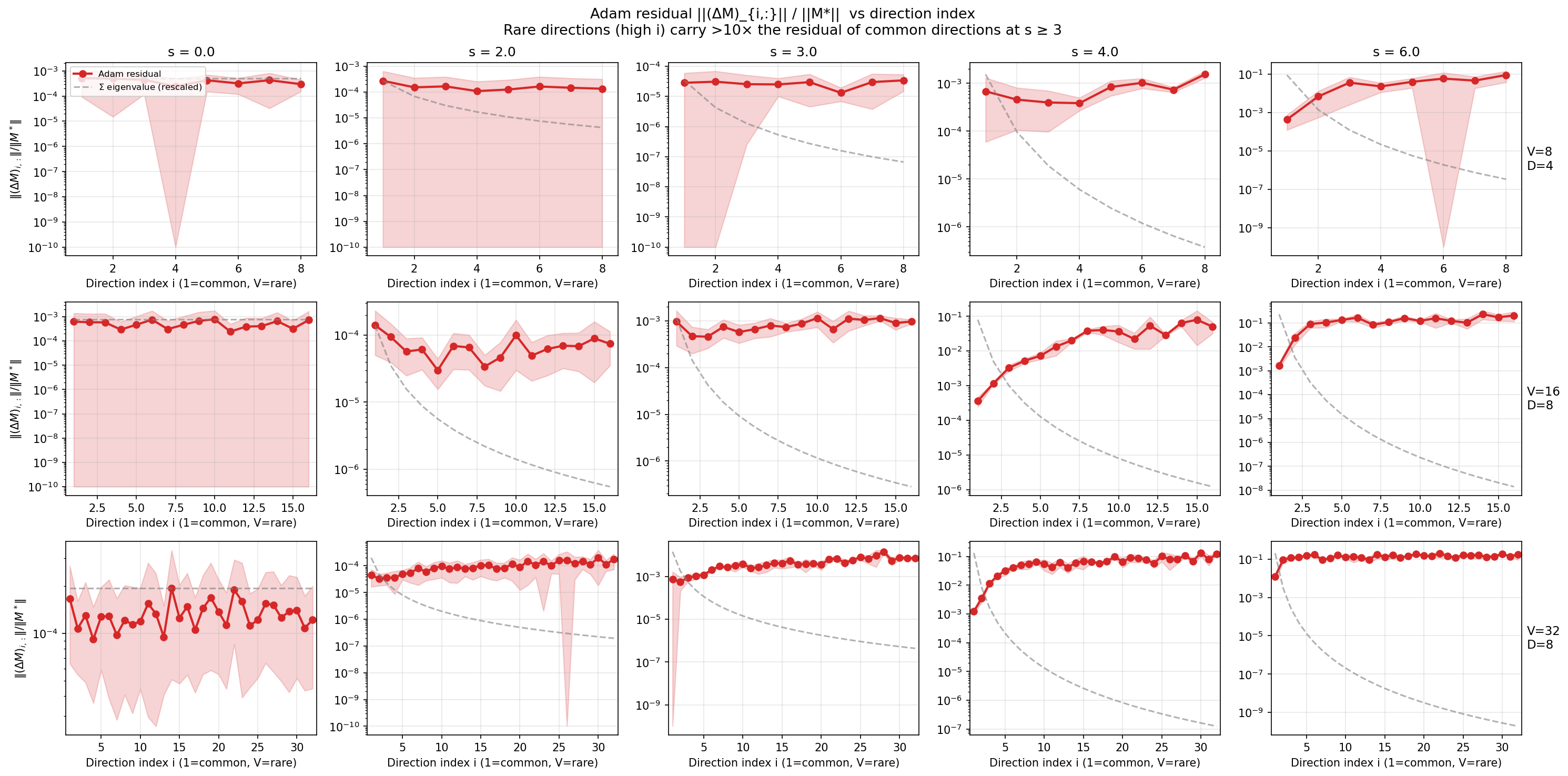}
\caption{Per-direction relative error $\|(\Delta M)_{i,:}\| / \|M^\ast\|$ across spectral index $i$ for Adam at $\varepsilon_0 = 10^{-8}$. At $s = 0$ errors are uniform across coordinates; for $s \geq 3$, error concentrates in the rarest directions (high $i$) by up to $25\times$ over common directions.}
\label{fig:residual}
\end{figure}

\begin{figure}[t]
\centering
\includegraphics[width=0.68\linewidth]{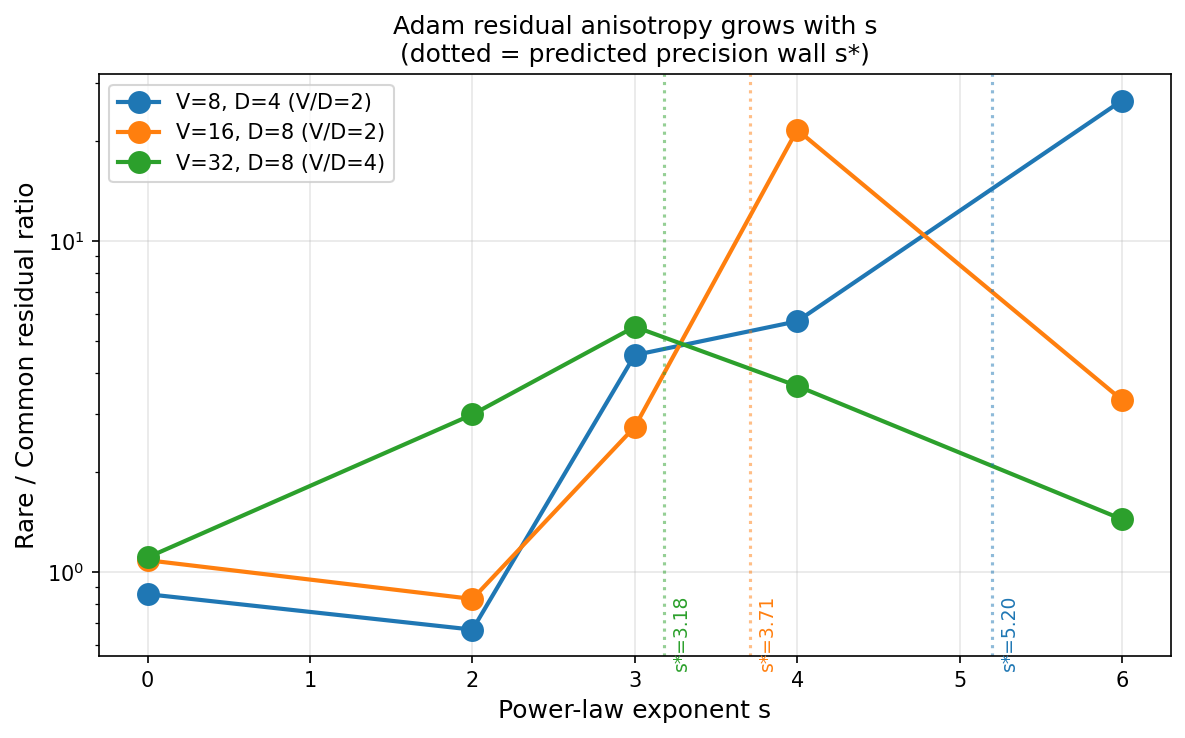}
\caption{Ratio of rare-direction to common-direction residual error vs.\ exponent $s$. Dotted vertical lines mark the predicted precision wall $\sstar(V,D)$ from~\eqref{eq:sstar} at $\varepsilon_0 = 10^{-8}$; above $\sstar$, the ratio turns over as common directions also begin to degrade.}
\label{fig:slowdown}
\end{figure}

Decomposing the recovery error $\Delta M = U_{\mathrm{adam}} U_{\mathrm{adam}}^\top - M^\ast$ in the $\Sig$ eigenbasis confirms the directional mechanism of Proposition~\ref{thm:wall}. Figure~\ref{fig:residual} shows that while isotropic sensing ($s=0$) yields uniform residuals across all $V$ coordinates, power-law sensing ($s \geq 3$) concentrates error in the high-index (rarest) directions by up to $25\times$. Figure~\ref{fig:slowdown} shows that this rare-to-common anisotropy ratio peaks near the predicted threshold $\sstar(V,D)$ before turning over once bulk directions also fail at high $s$. Below $\sstar(V,D)$, the step count required to reach $\relDM < 0.005$ grows with the condition number $\kappa(H_s) = \lambda_{\max}/\lambda_{\min} \propto V^{s+1}D^{s}$.

\section{Optimizer-Dependent Walls}
\label{sec:optwalls}

To test whether~\eqref{eq:sstar} specifically characterizes division-normalised methods, we compare eight optimizers at $(V,D) = (16,4)$ ($m = \mSixteenFour$) across $s \in \{0, 2, 3, 3.5, 4, 4.5, 5, 6\}$ with $20$ seeds per condition ($1280$ runs total, learning rates tuned at $s=0$, scored on aggregate $\relDM > 0.05$ at $60\text{k}$ steps).

\begin{figure}[t]
\centering
\includegraphics[width=0.92\linewidth]{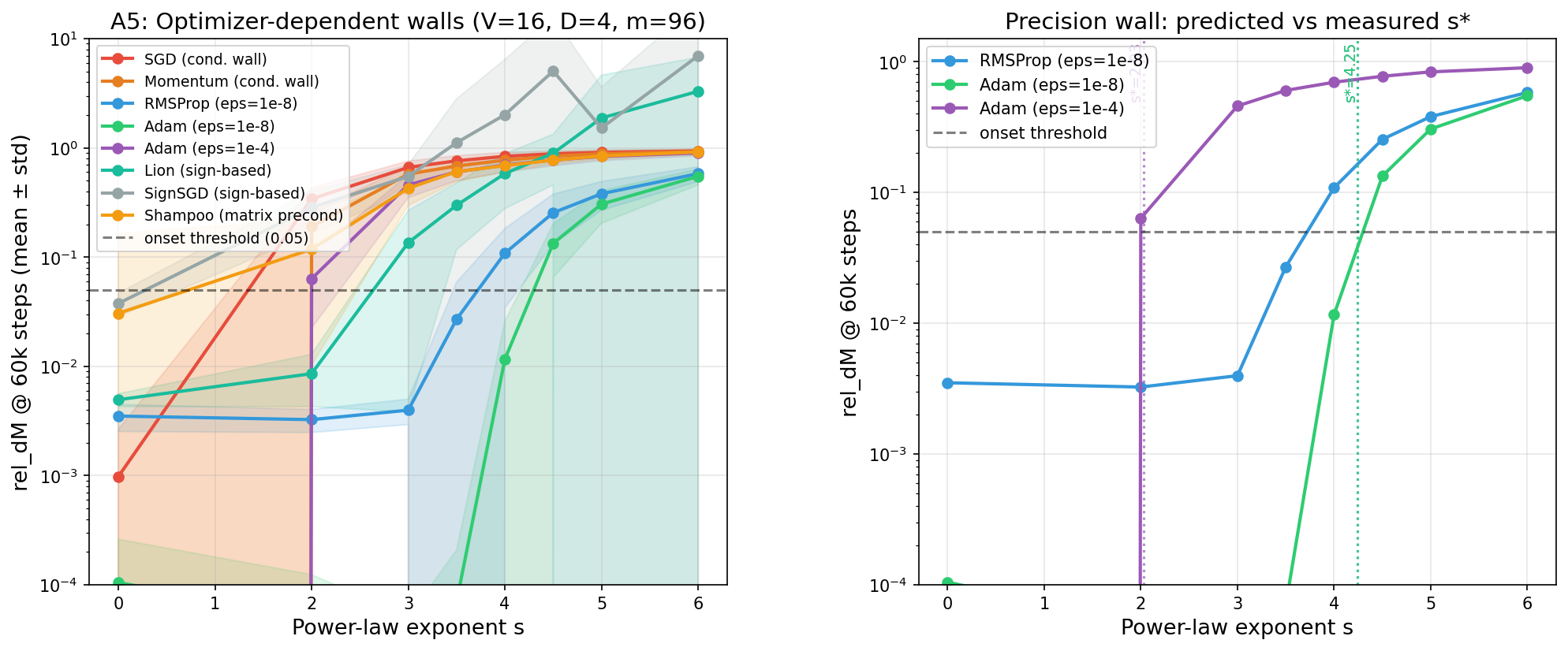}
\caption{Recovery error $\relDM$ at $60\text{k}$ steps across spectral exponent $s$ for eight optimizers at $(V,D)=(16,4)$. \emph{Left:} Three distinct failure regimes appear: sign-quantisation stall (SignSGD, Lion), conditioning stall (SGD, Momentum), and denominator-precision stall (Adam, RMSProp). \emph{Right:} Division-normalised optimizers with vertical dotted lines marking the predicted walls $\epsPredFour$ ($\varepsilon_0 = 10^{-4}$) and $\epsPredEight$ ($\varepsilon_0 = 10^{-8}$).}
\label{fig:a5}
\end{figure}

\subsection{Three distinct optimizer failure mechanisms}

Figure~\ref{fig:a5} separates first-order methods into three families:
\begin{enumerate}
    \item \textbf{Conditioning wall (SGD, Momentum).} Optimizers without per-coordinate division normalisation have no denominator floor $\epsilon_O$. Their convergence rate is limited directly by the Hessian condition number $\kappa(H_s) = \lambda_{\max}/\lambda_{\min} \propto V^{s+1}D^{s}$. Under a $60\text{k}$-step budget, SGD stalls at $\sstar = \optOnsetSgd$ and Momentum at $\sstar = \optOnsetMomentum$.
    \item \textbf{Denominator-precision wall (Adam, RMSProp, Adam with $\varepsilon_0 = 10^{-4}$).} Diagonal adaptive methods normalise updates by $\sqrt{v_t} + \varepsilon_0$, cancelling spectral conditioning until curvature drops below $\epsilon_O$, where they hit the precision wall $[\log(1/\epsilon_O) - \log V]/\log(VD) + \delta_O$.
    \item \textbf{Sign-quantisation wall (SignSGD, Lion~\citep{chen2023lion}).} Sign-based updates replace gradient magnitudes with $\mathrm{sign}(g)$, stalling early under anisotropic coupling ($\sstar = \optOnsetSignsgd$ for SignSGD; $\sstar = \optOnsetLion$ for Lion with momentum blending).
\end{enumerate}

\subsection{Clean $\varepsilon_0$ intervention at $(16,4)$ and the Shampoo counterexample}

Within the 20-seed optimizer sweep at $(16,4)$, holding the code path, tuned learning rate, and aggregate $\relDM$ metric fixed while varying only Adam's $\varepsilon_0$ between $10^{-8}$ and $10^{-4}$ provides a direct test of the $\log(1/\epsilon_O)/\log(VD)$ slope:
\begin{itemize}
    \item \textbf{Adam ($\varepsilon_0 = 10^{-8}$):} predicted $\sstar = \epsPredEight$, measured $\sstar = \epsOnsetEight$ (residual $\epsDeltaEight$, reflecting the coarser $s$-grid and per-optimizer tuned rate of this sweep compared to Table~\ref{tab:sstar}).
    \item \textbf{Adam ($\varepsilon_0 = 10^{-4}$):} predicted $\sstar = \epsPredFour$, measured $\sstar = \epsOnsetFour$ (residual $\epsDeltaFour$).
\end{itemize}
In the difference between the two arms, both $\delta_O$ and the protocol offset cancel: the measured shift is $\Delta\sstar_{\mathrm{measured}} = \epsOnsetEight - \epsOnsetFour = \epsShiftMeas$ against the parameter-free prediction $\Delta\sstar_{\mathrm{predicted}} = \ln(10^4)/\ln(64) = \epsShiftPred$ (ratio $\epsShiftRatio$).

\textbf{Why matrix preconditioning can worsen the wall: Shampoo.}
One might expect full-matrix block preconditioning via Shampoo~\citep{gupta2018shampoo} ($\epsilon = 10^{-8}$) to outperform diagonal Adam under heavy spectral skew. Instead, Shampoo fails much earlier, at $\sstar = \optOnsetShampoo$. Because Shampoo accumulates $P \mathbin{+}= g g^\top / V$ and applies $P^{-1/4}$, rare coordinates with vanishing gradient have $P_{ii} \approx \epsilon$, so the inverse fourth root multiplies rare-direction noise by $\epsilon^{-1/4} = 10^2$ instead of suppressing it.

\section{Extrapolation to Large-Vocabulary Heads and Discussion}
\label{sec:desert}
\label{sec:discussion}

\subsection{Order-of-magnitude extrapolation to large vocabulary size}

Because $\log V$ appears with a negative sign in the numerator of~\eqref{eq:sstar} as well as in the denominator $\log(VD)$, the precision wall drops steadily as $V$ grows. Consider the dimensions of a standard large-vocabulary output head, $(V, D) = (32000, 1024)$, at $\epsilon_O = 10^{-8}$:
\begin{equation}
\label{eq:lmhead}
\begin{split}
\sstar_{\mathrm{Adam}}
  &= \frac{\log(10^8) - \log(V)}{\log(V D)} + \deltaAdamThree \\
  &= \frac{18.421 - 10.373}{17.305} + \deltaAdamThree
   = 0.465 + \deltaAdamThree \approx \lmStar
\end{split}
\end{equation}
(where at one decimal place the numerator and denominator display as $\lmBase$ and $\lmLogVD$). At this scale, the calibrated lag $\delta_{\mathrm{Adam}} = \deltaAdam$ accounts for $\sim \lmShare\%$ of $\sstar$, so~\eqref{eq:lmhead} is an order-of-magnitude extrapolation across three decades in $V$ rather than a precise point forecast. Qualitatively, it indicates that at large $V$, $\sstar$ drops to $O(1)$---meaning that even Zipfian power-law spectra with $s \approx 1$ approach the regime where rare-direction curvature contends with the optimizer's denominator floor.

\subsection{Relation to prior work}

\textbf{Role of Adam's $\varepsilon_0$ and adaptive conditioning.}
\citet{everett2024adam} demonstrated empirically that fixed $\varepsilon_0$ causes gradient underflow across model scales and introduced Adam-atan2 to remove the hyperparameter. Our results supply an analytical and experimental phase boundary for this effect in a setting with a single tunable spectral exponent $s$: diagonal preconditioning neutralises spectral skew until $\lambda_{\min}^{\mathrm{ng}}$ reaches $\epsilon_O$. Similarly, \citet{kunstner2024zipf,kunstner2025adam} showed that Adam and sign descent dramatically outperform gradient descent under Zipfian class imbalance; Proposition~\ref{thm:wall} and Figure~\ref{fig:a5} identify the upper boundary $s = \sstar_O(V,D)$ where Adam's diagonal preconditioning advantage terminates.

\textbf{Power-law phase retrieval and random matrix landscapes.}
\citet{braun2025powerlaw} studied the convergence rate of rank-1 phase retrieval under power-law spectra, characterizing a slow spectral-tail learning phase. Here we study the rank-$D$ Burer--Monteiro setting and focus on the precision-limited termination threshold rather than asymptotic rates. Unlike spherical spin-glass complexity analyses via the Kac--Rice formula~\citep{auffinger2013random,auffinger2013complexity}, which require isotropic invariance to count critical points, the stagnation documented here occurs in a trap-free landscape and is governed by the local Hessian spectrum at $U^\ast$ relative to $\epsilon_O$.

\subsection{Limitations and practical implications}

\begin{enumerate}
    \item \textbf{Calibration scope, grid resolution, and metric sensitivity of the $D$-trend.} Although out-of-sample validation spans $V \in \{8, \ldots, 512\}$ and three octaves of rank $D \in \{4, 8, 16, 32\}$ (\S\ref{sec:grid}), five concrete limitations remain: (i) the detection lag $\delta_{\mathrm{Adam}} = \deltaAdam$ is calibrated solely on the four $V \leq 32, D \in \{4,8\}$ configurations of Table~\ref{tab:sstar}; (ii) the \gridNScored-cell multi-octave grid tests only Adam at fixed $\varepsilon_0 = 10^{-8}$ and learning rate $\eta = 10^{-2}$; (iii) the $s$-grid spacing on the multi-octave grid is $0.25$ (with a $\sim \gridRoundingSdRel\text{--}\gridRoundingSdRare$ rounding-level onset spread); (iv) the downward drift of $\Delta$ with $D$ holds in \gridRfiveNegRows\ of \gridRfiveTotRows\ rows on aggregate $\relDM$ but fails (2 of 4 rows) on the rarest-direction metric; and (v) empirical optimization runs extend only to $V = 512$, leaving the $(32000, 1024)$ calculation in~\eqref{eq:lmhead} as an untested analytical extrapolation.
    \item \textbf{Failure at rank $D = 2$.} All positive validation results require $D \geq 4$. At $D = 2$, $(16,2)$ walls much later than predicted ($\Delta = \hoDeltaSixteenTwo$) and reverses the fixed-$VD$ pair ordering, consistent with a two-column compression $N^\ast$ remaining better conditioned than the large-$D$ tail-mass law predicts.
    \item \textbf{Saturation of the $\varepsilon_0$ lever below $10^{-8}$.} Because Adam generates an internal dynamical floor $\epsilon_{\mathrm{dyn}} \approx 4.8\times 10^{-9}$ at $\eta = 10^{-2}$ (Appendix~\ref{app:companion}; \citealp{companion2026floor}), practitioners cannot push the precision wall to larger $s$ simply by setting $\varepsilon_0 = 10^{-12}$. Overcoming the wall above $\sstar$ requires reducing the learning-rate-induced chatter floor, improving measurement conditioning, or switching to second-order/SDP solvers within their identifiability limits.
\end{enumerate}

\section{Conclusion}
\label{sec:conclusion}

Under power-law measurement spectra, nonconvex Burer--Monteiro matrix sensing stalls not because spurious local minima emerge, but because curvature along the rarest measurement directions drops below the optimizer's effective denominator floor. Balancing the leading-order non-gauge Hessian scaling $V^{-(s+1)}D^{-s}$ against working precision $\epsilon_O$ yields the closed-form approximation $\sstar_O(V, D) = [\log(1/\epsilon_O) - \log V]/\log(VD) + \delta_O$. Calibrated with a single lag constant on four $D \geq 4$ configurations ($|\Delta| \leq \maxAbsDelta$), this boundary predicts held-out configurations across $V \in \{32, \ldots, 512\} \times D \in \{4, \ldots, 32\}$ (including all \gridRonePass\ of \gridNScored\ cells of the multi-octave Adam grid within $0.45$), captures the parameter-free doubling shifts across fixed-$VD$, fixed-$V$, and fixed-$D$ pairs, and separates denominator-precision walls from conditioning and sign-quantisation walls across eight optimizers.

\newpage
\appendix

\section{Companion Derivation: Dynamical Stall Floor at Vanishing Curvature}
\label{app:companion}

To make the $\varepsilon_0$-saturation discussion of \S\ref{sec:validation} self-contained without requiring the reader to consult~\citet{companion2026floor}, we summarize here the one-dimensional period-2 orbit derivation and the cross-mode chatter mechanism that sets Adam's denominator floor below $\varepsilon_0 = 10^{-8}$.

\textbf{Period-2 flip orbit on a quadratic mode.}
Consider warm-started Adam ($\hat{m}_t = m_t$, $\hat{v}_t = v_t$) with step size $\alpha > 0$, momentum parameters $\beta_1, \beta_2 \in (0,1)$, and denominator hyperparameter $\varepsilon_0 > 0$ on a scalar quadratic $f(x) = \frac{1}{2}\lambda x^2$ ($g_t = \lambda x_t$):
\[
m_t = \beta_1 m_{t-1} + (1-\beta_1)\lambda x_t, \qquad
v_t = \beta_2 v_{t-1} + (1-\beta_2)\lambda^2 x_t^2, \qquad
x_{t+1} = x_t - \alpha\frac{m_t}{\sqrt{v_t} + \varepsilon_0}.
\]
Seek a period-2 alternating orbit $x_{t+1} = -x_t$ with constant amplitude $|x_t| = A > 0$ and $m_t = -m_{t-1}$. Along such an orbit, $v_t$ is stationary at $v_t = (\lambda A)^2$, so $\sqrt{v_t} = \lambda A$. Stationarity of $m_t$ under alternating signs $x_t = -x_{t-1}$ gives $m_t = \frac{1-\beta_1}{1+\beta_1}\lambda x_t$. Substituting this into the update $x_{t+1} - x_t = -2x_t$ yields
\[
2A = \alpha \frac{1-\beta_1}{1+\beta_1}\frac{\lambda A}{\lambda A + \varepsilon_0} = 2A^\ast \frac{\lambda A}{\lambda A + \varepsilon_0}, \qquad A^\ast := \frac{\alpha(1-\beta_1)}{2(1+\beta_1)}.
\]
At $\beta_1 = 0.9$, $A^\ast = \alpha/38$, where $u_{c2} := 2(1+\beta_1)/(1-\beta_1) = 38$ is the adaptive edge-of-stability constant~\citep{cohen2022adaptive}. Defining the dimensionless curvature-to-precision ratio $u := \alpha\lambda/\varepsilon_0$ and solving $\lambda A + \varepsilon_0 = \lambda A^\ast$ gives the exact orbit amplitude law
\begin{equation}
\label{eq:app-au}
A(u) = A^\ast - \frac{\varepsilon_0}{\lambda} = A^\ast\Bigl(1 - \frac{u_{c2}}{u}\Bigr) = \frac{\alpha}{38}\Bigl(1 - \frac{38}{u}\Bigr).
\end{equation}
A positive-amplitude period-2 orbit exists if and only if $u > u_{c2} = 38$, i.e.,
\[
\lambda > \lambda_{\mathrm{stall}} := \frac{38\varepsilon_0}{\alpha}.
\]
For $u < 38$ ($\lambda < \lambda_{\mathrm{stall}}$), the denominator is pinned by $\varepsilon_0$ and the linearised 1-D map contracts toward $x = 0$.

\textbf{Cross-mode gradient leakage and the self-generated floor.}
In the coupled $V \times D$ matrix sensing problem, high-curvature (common) modes sit far above $\lambda_{\mathrm{stall}}$ ($u \gg 38$) and oscillate persistently near the period-2 attractor of amplitude $\sim A^\ast \propto \alpha(1-\beta_1)/(1+\beta_1)$. Because the measurement matrices $A_k = \Sig^{1/2} G_{\mathrm{sym},k} \Sig^{1/2}$ couple all rows of $U$, this fast-mode oscillation leaks gradient fluctuations of scale $\kappa\alpha$ (at $(V,D) = (8,4)$, $\kappa \sim 5.65\times 10^{-6}$, so $\kappa\alpha \approx 5.65\times 10^{-8}$ at $\alpha = 10^{-2}$) into the rarest row $j = V$. Those leaked fluctuations feed the rare row's second-moment accumulator $v_V$, pinning $\sqrt{\hat{v}_V} \approx \kappa\alpha$ orders of magnitude above a forced $\varepsilon_0 = 10^{-10}$ or $10^{-12}$. Consequently, the effective denominator floor of Adam at $\alpha = 10^{-2}$ is governed by $\max(\varepsilon_0, \epsilon_{\mathrm{dyn}})$ with fitted wall equivalent $\epsilon_{\mathrm{dyn}} \approx 4.8\times 10^{-9}$, explaining why the $\varepsilon_0$-sweep in \S\ref{sec:validation} retains roughly half its slope across $10^{-6} \to 10^{-8}$ and saturates below $10^{-8}$.

\section{Pre-registration Ledger and Protocol History}
\label{app:ledger}
\sloppy

This appendix records the full protocol history, the correction from the initial candidate scaling law to Proposition~\ref{thm:eigval}, the complete adjudication of the six pre-registered criteria (C1--C6) and three pre-registered risks in \texttt{scripts/validation/preregistration.json}, and the artifact provenance of every macro in \texttt{numbers.tex}.

\subsection{Correction from Law A to Law B and like-for-like spectral fits}

An initial version of the derivation approximated $\nu_{\min}(N^\ast) \sim V^{-s}$ by assuming that $\sum_{i=1}^V \lambda_i(\Sig) u_i u_i^\top$ concentrates on its smallest weight $\lambda_V(\Sig)$, which yielded the candidate scaling $\lambda_{\min}^{\mathrm{ng}} \propto (VD)^{-(s+1)}$ (``Law A'', wall formula $\sstar = \log(1/\epsilon_O)/\log(VD) - 1 + \delta_O$). As shown in \S\ref{sec:hessian}, $\nu_{\min}(N^\ast)$ is independent of $V$ and governed by $D$, yielding the leading-order form $\lambda_{\min}^{\mathrm{ng}} \propto V^{-(s+1)}D^{-s}$ (``Law B'', Proposition~\ref{thm:eigval}).

Recalibrating $\delta_{\mathrm{Adam}}$ on the four rows of Table~\ref{tab:sstar} under Law B gives $\delta_{\mathrm{Adam}} = \deltaAdam$ ($\deltaAdamThree$) with in-sample RMS $\rmsCalibB$, compared to $\delta = \deltaCalibA$ and RMS $\rmsCalibA$ under Law A. On the parameter-free differential test $\sstar(16,4) - \sstar(16,8) = \dtestMeas$ of Remark~\ref{rem:ddep}, Law A predicts $\dtestA$ (error $\dtestErrA$), whereas the Law B scaling law predicts $\dtestB$ (error $\dtestErrB$). Across all $\onsetNTotal$ configurations with measured onsets (spanning $\sstar \in [\onsetLo, \onsetHi]$), $\onsetNAbove$ of $\onsetNTotal$ lie above the finite-$m$ spectral crossover $s = \dirCross$. Table~\ref{tab:spectral-lfl} reports the three-parameter like-for-like spectral RMS in $\ln\lambda$ across all four cell sets produced by \texttt{scripts/spectral\_evidence.py}.

\begin{table}[h]
\centering\footnotesize
\setlength{\tabcolsep}{4pt}
\caption{Like-for-like three-parameter regression RMS in $\ln\lambda_{\min}^{\mathrm{ng}}$ across spectral cell sets (\texttt{scripts\slsh spectral\us evidence.py} $\to$ \texttt{data\slsh spectral\us evidence.json}).}
\label{tab:spectral-lfl}
\begin{tabular}{lrccc}
\toprule
cell set & $n$ & Law A: $(VD)^{-(s+1)}$ & Law B: $V^{-(s+1)}D^{-s}$ & Mid: $V^{-(s+1)}D^{-(s+1/2)}$ \\
\midrule
all exact cells & \lflAllN & \textbf{\lflAllA} & \lflAllB & \lflAllMid \\
exact, $s \geq 4$ & \lflDeepN & \textbf{\lflDeepA} & \lflDeepB & \lflDeepMid \\
finite-$m$ direct ($\lambda_{\min}/\lambda_{\max} > 10^{-12}$) & \lflDirectN & \lflDirectA & \lflDirectB & \textbf{\lflDirectMid} \\
validation $(V,D)$ grid & \lflGridN & \lflGridA & \textbf{\lflGridB} & \lflGridMid \\
\bottomrule
\end{tabular}
\end{table}

\textbf{Finite-size scaling refutation on $\delta_O$.}
A finite-size scaling ansatz $\delta(V,D) = \delta_\infty + c\,(VD)^{-1/\nu}$ was tested on the four calibration rows by \texttt{scripts\slsh fss\us probe.py} (output \texttt{scratch\slsh fss\us probe.json}, sha16 \texttt{3422fa5bb2936ad4}) and is refuted: $R^2 = 0.05$ against a pre-registered bar of 0.95, with $\nu = -4.1$ outside the physical range and $R^2 \in [0.009, 0.249]$ under $\pm 0.005$ display rounding.

\textbf{Provenance note on the $(16,4)$ calibration row.}
The pre-registration file \texttt{scripts\slsh validation\slsh preregistration.json} records \texttt{runner\us sha256} prefix \texttt{3ccc2577bf288c24} for the script that generated \texttt{scripts\slsh validation\slsh genuine\us 16\us 4\us 10seed.jsonl}. The data file itself is archived and verified, whereas the standalone runner script matching that hash was not separately archived in the repository.

\subsection{Adjudication of pre-registered criteria C1--C6 and Risks 1--3}

The file \texttt{scripts\slsh validation\slsh preregistration.json} (registered 2026-08-06 prior to the validation runs, initially written against Law A, later compared against $\sstar_O(V, D) = [\log(1/\epsilon_O) - \log V]/\log(VD) + \delta_O$) defines six criteria C1--C6 and three named risks, scored by \texttt{scripts\slsh criteria\us audit.py} (\texttt{data\slsh criteria\us audit.json}) and \texttt{scripts\slsh validation\slsh verdict.py}:

\begin{itemize}
    \item \emph{C1} (aggregate onset within $\pm 0.45$ of base $+\,\regDelta{}$ on $\geq 5/6$ configurations): \textbf{PASS} ($5/6$ under both laws). The pre-registered constant $\delta = \regDelta{}$ sits between the held-out $D \geq 4$ aggregate lag under Law A ($\aggLagA{}$, difference $\regDeltaErr{}$) and Law B ($\aggLagB{}$, law shift $\lawLagShift{}$). On the same held-out $D \geq 4$ cells, the rarest-direction lag is $\rareLagA{}$ under Law A and $\rareLagB{}$ under Law B (a metric shift of $\metricLagShift{}$). Because the $\pm 0.45$ band is wide enough to absorb a $\regDeltaErr{}$ offset discrepancy, C1 is a coarse consistency check rather than a tight discriminator.
    \item \emph{C2} (rarest-direction onset within $\pm 0.40$ of the zero-lag base, zero fitted constants, $\geq 5/6$): \textbf{FAIL} ($0/6$ as registered under Law A, $2/6$ under Law B). This criterion tested $\delta_O = 0$ and fails precisely because the empirical detection lag is $\delta_{\mathrm{Adam}} \approx \deltaAdam$ rather than zero.
    \item \emph{C3} (fixed-$VD$ pairs within $|\text{diff}| \leq 0.35$, originally registered to test a zero difference under Law A): \textbf{2 of 3 pairs within band} ($\pairMeasI{}$ at $VD=128$ and $\pairMeasII{}$ at $VD=256$, while the $D=2$ pair at $VD=32$ misses at $\pairMeasIII{}$). Under Law B, paired bootstrap (\texttt{scripts\slsh pair\us bootstrap.py}) shows $\Pr(\text{diff} < 0) = 0.993$ and $1.000$ on the two $D \geq 4$ pairs, supporting the negative $-\log 2/\log(VD)$ shift.
    \item \emph{C4} ($\varepsilon_0$-slope within $\pm 20\%$ of $\log 10/\log(VD)$, and no saturation in fp64 down to $10^{-12}$): \textbf{FAIL on the non-saturation clause} ($1/2$ on the $10^{-4}\to 10^{-8}$ slope: $16\%$ error at $(16,8)$, $30\%$ error at $(8,4)$; and the slope collapses to $0.038$ and $0.021$ per decade below $10^{-8}$ due to the dynamical floor of Appendix~\ref{app:companion}).
    \item \emph{C5} (spectral crossings in \texttt{scripts\slsh validation\slsh leg7\us spectral.jsonl} within $\pm 0.20$ of the zero-constant base): \textbf{FAIL as registered with zero constants under Law B} ($\specCFivePassB$ of $\specNB$ under Law B vs.\ $\specCFivePassA$ of $\specNA$ under Law A, because Law B's raw crossing carries a common offset of $\specOffsetB$). With one common offset fitted post hoc, Law B has smaller spread across all $\specNB$ crossings ($\pm \specDevB$ vs.\ $\pm \specDevA$) and reduces the $(16,2)$ spectral bias from $[\specDtwoLoA, \specDtwoHiA]$ to $[\specDtwoLoB, \specDtwoHiB]$.
    \item \emph{C6} ($200\text{k}$-step vs.\ $60\text{k}$-step onset shift $< 0.30$ at $\varepsilon_0 = 10^{-12}$): \textbf{UNSCORED / CENSORED.} The $200\text{k}$-step run exists in \texttt{scripts\slsh validation\slsh leg6\us budget.jsonl} ($21$ rows, $n=3$ seeds at $(8,4)$ for $s \in \{6.5,\ldots,8.0\}$ and $(16,8)$ for $s \in \{4.5,\ldots,5.5\}$), but the lowest sampled $s$ already exceeds the $0.05$ threshold ($0.100$ at $s=6.5$ for $(8,4)$; $0.158$ at $s=4.5$ for $(16,8)$), censoring the onset from below, and no C6 scorer was implemented in \texttt{verdict.py}.
\end{itemize}

\textbf{Outcomes of the three pre-registered risks.}
\emph{Risk 1} ($D = 2$ lies outside the $D \geq 4$ regime): \textbf{Materialised} at $(16,2)$ ($\Delta = \hoDeltaSixteenTwo$), motivating the explicit $D \geq 4$ domain restriction.
\emph{Risk 2} ($\varepsilon_0 \leq 10^{-10}$ onsets may be $60\text{k}$-budget-limited, delegated to C6): \textbf{Open}, since C6 is censored from below on the sampled grid.
\emph{Risk 3} ($\delta_O$ varies with metric and law, separated by C1 vs.\ C2): \textbf{Materialised}, accounted for by explicit metric-specific lag calibration in Definition~\ref{def:lag}.

\subsection{Pre-registered 16-cell Adam grid: protocol, amendments, and sensitivity}
\label{app:grid-ledger}

The \gridNScored-cell Adam grid of \S\ref{sec:grid} ($V \in \{64, 128, 256, 512\} \times D \in \{4, 8, 16, 32\}$) was pre-registered prior to execution in \texttt{scripts\slsh validation\slsh supp\us 2026-10\slsh prereg\us supp\us 2026-10-28\us part2.json} and scored by \texttt{scripts\slsh validation\slsh grid\us score.py} ($\to$ \texttt{grid\us verdict.json}).

\textbf{Registered protocol and read definitions.}
The prediction uses $\delta_{\mathrm{Adam}} = \deltaAdamThree$ frozen from Table~\ref{tab:sstar} (no grid cell overlaps the calibration set), $\varepsilon_0 = 10^{-8}$, $\eta = 10^{-2}$, $\beta = (0.9, 0.999)$, $m = 1.5VD$, float64 arithmetic, $5$ seeds ($0\text{--}4$), and $60\text{k}$ steps (with checkpoints at $20\text{k}$ and $60\text{k}$). Each cell's $s$-grid consists of $8$ points $s = c + 0.25 k$ ($k = -3, \ldots, +4$) around the nearest quarter-integer $c$ to the predicted $\sstar(V,D)$ ($40$ trajectories per cell; all \gridNScored\ cells completed with zero censored onsets). Five reads and one conversion gate were registered:
\begin{itemize}
    \item \emph{R1 (cell-level residuals):} $\Delta = \sstar_{\mathrm{meas}}(\relDM) - \sstar_{\mathrm{pred}}$ at $60\text{k}$ steps; pre-registered verdict ``SUPPORTED'' iff $|\Delta| \leq 0.45$ in $\geq 80\%$ of scored cells, with the count within $\pm 0.15$ and the $(32,4)$ one-ULP rounding-level onset standard deviation ($\gridRoundingSdRel$ on $\relDM$, $\gridRoundingSdRare$ on rarest) printed alongside. Verdict: \textbf{SUPPORTED} (\gridRonePass\ of \gridNScored\ within $0.45$, \gridRoneTight\ within $0.15$, largest $|\Delta| = \gridRoneMaxAbs$ at $(128,4)$).
    \item \emph{R2 (fixed-$V$, $D \to 2D$ shifts, \gridRtwoN\ pairs) and R3 (fixed-$D$, $V \to 2V$ shifts, \gridRthreeN\ pairs):} measured onset shift minus the parameter-free predicted shift ($\delta_{\mathrm{Adam}}$ cancels), counted against a $\pm 0.35$ band. Verdict: \gridRtwoPass\ of \gridRtwoN\ pairs within $0.35$ for R2 (largest deviation $\gridRtwoMaxDev$, at $(512,16)\to(512,32)$) and \gridRthreePass\ of \gridRthreeN\ pairs within $0.35$ for R3 (largest deviation $\gridRthreeMaxDev$, at $(64,4)\to(128,4)$).
    \item \emph{R4 (descriptive rarest-direction offset):} $\sstar_{\mathrm{meas}}(\text{rarest}) - [\log(1/\varepsilon_0) - \log V]/\log(VD)$ (zero lag), reported descriptively in \texttt{grid\us verdict.json} (from $+0.212$ at $(256,4)$ to $+0.589$ at $(64,8)$).
    \item \emph{R5 (row-wise trend with $D$):} Spearman rank correlation $\rho(\Delta, D)$ across $D \in \{4, 8, 16, 32\}$ in each fixed-$V$ row; pre-registered verdict ``HOLDS'' iff $\rho(\Delta, D) < 0$ in $\geq 3$ of $4$ rows (reflecting the $\approx \frac{1}{2}\log D/\log(VD)$ downward shift from the empirical $D^{-(s+1/2)}$ spectral exponent). Verdict: \textbf{HOLDS} on the registered $\relDM$ metric (\gridRfiveNegRows\ of \gridRfiveTotRows\ rows negative: $V=64$ gives $\gridRfiveRhoSixtyFour$, $V=128$ gives $\gridRfiveRhoOneTwoEight$, $V=256$ gives $\gridRfiveRhoTwoFiveSix$, $V=512$ gives $\gridRfiveRhoFiveOneTwo$).
    \item \emph{F32 conversion gate at $(512,32)$:} Cell $(512,32)$ stores the upper triangle of $G_{\mathrm{sym}}$ in float32 (defining the instance as $G_{\mathrm{sym}}$ rounded to float32) and widens entries in-kernel to float64 arithmetic. Pre-registered gate \texttt{f32\us conversion\us gate.py} (\texttt{f32\us gate.json}) verified bit-identical forward and backward outputs between float32-stored and float64-stored instances of the same rounded tensor at $(256,8)$ and $(512,4)$ (\textbf{PASS}), admitting $(512,32)$ to scoring (onset $\gridFthirtyTwoOnset$ vs.\ predicted $\gridFthirtyTwoLawB$, $\Delta = \gridDeltaFiveOneTwoThirtyTwo$, $|\Delta| = \gridFthirtyTwoAbsDelta$).
\end{itemize}

\textbf{Protocol amendments and checkpoint resumes.}
Before any $(512,32)$ trajectory was run, Amendment~1 (\texttt{prereg\us supp\us 2026-10-28\us part2\us amend1.json}) changed the float32 finite/subnormal validity check to scan $256$-row chunks rather than allocating a full-tensor copy, and Amendment~2 (\texttt{prereg\us supp\us 2026-10-28\us part2\us amend2.json}, \texttt{amend2\slsh big\us offset\us check.json}) cast the forward and backward kernel row offsets from \texttt{int32} to \texttt{int64} because $m \times T = 3{,}227{,}516{,}928 > 2^{31}$ only at $(512,32)$ (verified by bit-identical replay on test cells and agreement within $2.8\times 10^{-14}$ against a float64 reference at $(512,32)$). Two individual seed runs---$(512,16)$ seed 3 from step $50\text{k}$ and $(512,32)$ seed 4 from step $5\text{k}$---experienced an undetermined GPU memory access fault and were resumed to completion from their saved checkpoints under the pre-registered restart rule.

\textbf{Post-hoc sensitivity axes (not registered reads).}
Evaluating three sensitivity axes on the same trajectory rows shows:
\begin{enumerate}
    \item \emph{Leave-one-seed-out sensitivity:} Dropping one seed at a time across the five 4-seed folds leaves every count and verdict unchanged (16 of 16 within $0.45$, 9--11 within $0.15$, max $|\Delta| \in [0.219, 0.301]$; R2 and R3 remain 12 of 12 in every fold; R5 has $\geq 3$ of 4 negative rows in every fold, and 4 of 4 when seed 1 is omitted).
    \item \emph{Metric sensitivity (rarest-direction metric at $60\text{k}$ steps):} Scoring the rarest-direction onset against the same frozen prediction preserves R1 (16 of 16 within $0.45$, 11 within $0.15$, max $|\Delta| = 0.270$) and R3 (12 of 12 within $0.35$, max deviation $0.262$), reduces R2 to 11 of 12 pairs within $0.35$ ($(256,4)\to(256,8)$ deviates by $0.352$), and \textbf{causes R5 to fail}: only 2 of 4 rows have $\rho(\Delta_{\mathrm{rarest}}, D) < 0$ ($V=64$: $-0.8$, $V=128$: $+0.8$, $V=256$: $+0.2$, $V=512$: $-0.4$).
    \item \emph{Checkpoint sensitivity ($20\text{k}$ steps):} At the $20\text{k}$-step checkpoint, R1 (16 of 16 within $0.45$, 4 within $0.15$, max $|\Delta| = 0.384$), R2 (12 of 12, max deviation $0.104$), R3 (12 of 12, max deviation $0.139$), and R5 (3 of 4 rows negative: $V=64$ has $+0.4$, $V=128$ has $-0.8$, $V=256$ has $-0.4$, $V=512$ has $-0.8$) all hold, but the single non-negative row in R5 shifts from $V = 128$ at $60\text{k}$ to $V = 64$ at $20\text{k}$. As an exploratory observation, all 16 cell residuals are negative at both $60\text{k}$ (mean $\gridMeanDelta$) and $20\text{k}$ (mean $-0.220$).
\end{enumerate}

\subsection{Artifact provenance for generated macros}

All numerical macros in \texttt{numbers.tex} are regenerated deterministically by \texttt{scripts\slsh lawb\us numbers.py} (which reads \texttt{data\slsh criteria\us audit.json} from \texttt{scripts\slsh criteria\us audit.py}, \texttt{data\slsh spectral\us evidence.json} from \texttt{scripts\slsh spectral\us evidence.py}, and \texttt{scripts\slsh validation\slsh supp\us 2026-10\slsh grid\us verdict.json} from \texttt{scripts\slsh validation\slsh grid\us score.py}):
\begin{itemize}
    \item Table~\ref{tab:sstar} calibration rows: \texttt{scripts\slsh d1c\us V8\us complete.jsonl}, \texttt{scripts\slsh d1c\us V16\us 10seed.jsonl}, \texttt{scripts\slsh d1c\us V32\us 10seed.jsonl}, and \texttt{scripts\slsh validation\slsh genuine\us 16\us 4\us 10seed.jsonl}.
    \item Table~\ref{tab:heldout} held-out and pair rows: \texttt{scripts\slsh validation\slsh leg1\us VD32.jsonl}, \texttt{scripts\slsh validation\slsh leg2\us VD128.jsonl}, \texttt{scripts\slsh validation\slsh leg3a\us 32\us 8.jsonl}, and \texttt{scripts\slsh validation\slsh leg3b\us 64\us 4.jsonl}.
    \item Table~\ref{tab:grid} and \S\ref{sec:grid} 16-cell Adam grid: \texttt{scripts\slsh validation\slsh supp\us 2026-10\slsh grid\slsh grid\us V*\us D*.jsonl}, \texttt{f32\us gate.json}, and \texttt{grid\us verdict.json}.
    \item Section~\ref{sec:validation} mixed-precision and $\varepsilon_0$ sweeps: \texttt{scripts\slsh d1\us mixed\us precision.jsonl} (\texttt{scripts\slsh d1\us mp\us score.py}) and \texttt{scripts\slsh d1\us eps\us sweep\us complete.jsonl}.
    \item Section~\ref{sec:optwalls} 8-optimizer sweep: \texttt{scripts\slsh results\us a5\us optimizer\us walls.jsonl}.
    \item Spectral crossings and fits: \texttt{scripts\slsh validation\slsh leg7\us spectral.jsonl} and \texttt{scripts\slsh spectral\us evidence.py}.
\end{itemize}

\end{document}